\documentclass[letterpaper]{article} 
\usepackage{aaai25} 

\usepackage{algorithm}
\usepackage{algorithmic}

\usepackage{times}  
\usepackage{helvet}  
\usepackage{courier}  
\usepackage[hyphens]{url}  
\usepackage{graphicx} 
\usepackage{natbib}  
\usepackage{caption} 
\usepackage{subcaption}
\usepackage[bottom]{footmisc}
\usepackage{indentfirst} 
\usepackage{stmaryrd}
\usepackage{endnotes}
\usepackage{verbatim}
\usepackage{tikz}
\usepackage{threeparttable}
\usepackage{comment}
\usepackage{algorithm}
\usepackage{algorithmic}
\usepackage{multirow}
\usepackage{makecell}
\usepackage{chngcntr}
\usepackage{dsfont}
\usepackage{enumitem}
\usepackage{array}
\usepackage{booktabs}       
\usepackage{amsfonts}       
\usepackage{nicefrac}       
\usepackage{microtype}      
\usepackage{xcolor}         
\usepackage{algorithm}
\usepackage{algorithmic}
\usepackage{graphicx,float}
\usepackage{amsmath,amsthm,amssymb,color}
\usepackage{algorithm}
\usepackage{algorithmic}

\newtheorem{lemma}{Lemma}
\newtheorem{proposition}{Proposition}

\newcommand{\cA}{{\cal A}}
\newcommand{\cB}{{\cal B}}

\newcommand{\cD}{{\cal D}}
\newcommand{\cE}{{\cal E}}
\newcommand{\cF}{{\cal F}}

\newcommand{\cS}{{\cal S}}

\newcommand{\cU}{{\cal U}}

\newcommand{\cW}{{\cal W}}
\newcommand{\cX}{{\cal X}}
\newcommand{\cY}{{\cal Y}}

\def\E{\mathbb{E}}
\def\P{\mathbb{P}}
\def\RR{\mathbb{R}}
\def\R{\mathbb{R}}

\def\mat{\hbox{\rm mat}}

\usepackage{newfloat}
\usepackage{listings}
\DeclareCaptionStyle{ruled}{labelfont=normalfont,labelsep=colon,strut=off} 
\floatstyle{ruled}
\newfloat{listing}{tb}{lst}{}
\floatname{listing}{Listing}
\title{Factor Augmented Tensor-on-Tensor Neural Networks}
\author{
    Guanhao Zhou,
    Yuefeng Han,
    Xiufan Yu
}
\affiliations{
    Department of Applied and Computational Mathematics and Statistics \\ 
    University of Notre Dame \\
    gzhou4, yuefeng.han, xiufan.yu@nd.edu
}

\usepackage{bibentry}

\begin{document}

\maketitle

\begin{abstract}
This paper studies the prediction task of tensor-on-tensor regression in which both covariates and responses are multi-dimensional arrays (a.k.a., tensors) across time with arbitrary tensor order and data dimension. Existing methods either focused on linear models without accounting for possibly nonlinear relationships between covariates and responses, or directly employed black-box deep learning algorithms that failed to utilize the inherent tensor structure. In this work, we propose a Factor Augmented Tensor-on-Tensor Neural Network (FATTNN) that integrates tensor factor models into deep neural networks. We begin with summarizing and extracting useful predictive information (represented by the ``factor tensor'') from the complex structured tensor covariates, and then proceed with the prediction task using the estimated factor tensor as input of a temporal convolutional neural network. The proposed methods effectively handle nonlinearity between complex data structures, and improve over traditional statistical models and conventional deep learning approaches in both prediction accuracy and computational cost. By leveraging tensor factor models, our proposed methods exploit the underlying latent factor structure to enhance the prediction, and in the meantime, drastically reduce the data dimensionality that speeds up the computation. The empirical performances of our proposed methods are demonstrated via simulation studies and real-world applications to three public datasets. Numerical results show that our proposed algorithms achieve substantial increases in prediction accuracy and significant reductions in computational time compared to benchmark methods.
\end{abstract}

%

\section{Introduction} \label{sec:introduction}

Motivated by a wide range of industrial and scientific applications, forecasting tasks using one tensor series to predict another have become increasingly prevalent in various fields. In finance, tensor models are used to forecast future stock prices by analyzing multi-dimensional stock attributes over time \citep{tran2017tensor}. In meteorology, tensor models predict future weather conditions using historical data formatted as multi-dimensional tensors (latitude, longitude, time, etc.) \citep{bilgin2021tent}. In neuroscience, tensor models are employed to process medical imaging like MRI/fMRI scans \citep{wei2023tensor}. 
Though predictive modeling of tensors has emerged rapidly over the past decade, most existing studies primarily focus on scenarios when either covariates or responses are tensors, such as scalar-on-tensor regression \citep{guo2011tensor,zhou2013tensor,wimalawarne2016theoretical,li2018tucker,ahmed2020tensor}, matrix-on-tensor regression \citep{hoff2015multilinear,kossaifi2020tensor}, and tensor-on-vector regression \citep{li2017parsimonious,sun2017store}. 
\textbf{Tensor-on-tensor}\footnote{Following the literature, we distinguish ``tensor-on-tensor regression'' from ``tensor regression''. ``Tensor regression'' refers to regression models with tensor covariates and scalar responses. ``Tensor-on-tensor regression'' refers to regression models with both tensor covariates and tensor responses. As a matter of fact, tensor-on-tensor regression encompasses tensor regression.} \textbf{prediction}, when both covariates and responses are tensors, 
remains a challenging task due to the inherent complexities of tensor structures. 

One commonly used strategy for tensor-on-tensor modeling is to flatten the tensors into matrices (or even vectors) \citep{kilmer2011factorization,choi2014dfacto} so that matrix-based analytical tools become applicable. However, flattening can lead to a loss of spatial or temporal information, especially for datasets that are inherently dependent on data structures. For instance, in image processing, the relative positions of pixels carry important information about object shapes, which is lost when the image is flattened. 

Alternatively, there is a handful of recent works on tensor-on-tensor regression models that do not adopt flattening but deal with the tensor structures directly \cite{lock2018tensor,liu2020low,gahrooei2021multiple,luo2022tensor}. 
However, these methods posit a linear prediction model between covariates and responses, limiting their capacity to capture potential nonlinear relationships.

Another emerging trend in tensor-on-tensor prediction involves neural networks, such as recurrent neural networks (RNN), convolutional neural networks (CNN), recurrent convolutional neural networks (RCNN), and temporal convolutional networks (TCN), to name a few. Deep neural network is a powerful tool in prediction tasks owing to its ability to learn intricate patterns from complex datasets. Structured with multiple layers of interconnected neurons, neural networks excel in capturing nonlinearity in data. Nevertheless, they often operate in a black-box manner and typically require extensive computational resources. 

In this work, we develop a Factor Augmented Tensor-on-Tensor Neural Network (FATTNN) for forecasting a sequence of tensor responses using tensor covariates of arbitrary tensor order and data dimensions. One key innovation of our method is the integration of tensor factor models into deep neural networks for tensor-on-tensor regression. Factor models have been a widely utilized tool in matrix/vector data analysis \citep{bai2002determining,stock2002forecasting,pan2008modelling,lam2012factor} for understanding common dependence among multi-dimensional covariates. In recent years, researchers have advanced the methodology of factor models to the context of tensor time series \citep{chen2022analysis,chen2022factor,han2020tensor,han2021cp,han2022rank,han2022tensor,chen2024estimation,yu2024dynamic}. In our models, we borrow the strengths from tensor factor models, offering a more general approach than traditional vector factor
analysis for prediction \citep{sen2019think,fan2017sufficient,yu2022nonparametric}. 

Our contributions can be summarized in three folds. 
\begin{itemize}
    \item We propose a FATTNN model for tensor-on-tensor prediction that integrates tensor factor models into deep neural networks. Via a tensor factor model, FATTNN exploits the latent factor structures among the complex structured tensor covariates, and extracts useful predictive information for subsequent modeling. The utilization of the factor model reduces the data dimension drastically while preserving the tensorial data structures, contributing to substantial increases in prediction accuracy and significant reductions in computational time.  
    \item In addition to the improved prediction accuracy and reduced computation cost, FATTNN offers several advantages over existing methods. First, it preserves the inherent tensor structure without any flattening or tensor contraction operations, preventing the possible loss of spatial or temporal information. Second, the utilization of neural networks equips FATTNN with exceptional capability in modeling nonlinearity in response-covariate relationships. 
    \item FATTNN provides a comprehensive framework for the general supervised learning question of predicting tensor responses using tensor covariates. Although initially introduced under the problem settings of tensor-on-tensor time series forecasting, the proposed framework is not limited to time series data, but can accommodate a variety of data types. With different types of factor models and choices of neural network architectures tailored to specific characteristics of the data, FATTNN can be naturally adapted to tensor-on-tensor regression for various data types. 
\end{itemize}

\section{Related Work}
Neural networks are powerful tools for processing complex tensor data. Early works involve RNN and its variations (e.g., Long Short-Term Memory (LSTM) networks) to accommodate the temporal dependence, and CNN to capture the temporal and spatial relationships underlying the data structures. Recently, more advanced architectures have been proposed, integrating the strengths of traditional models to capture the increasingly complex data dependence in a more powerful and efficient manner. Notable examples include convolutional LSTM (ConvLSTM) \citep{shi2015convolutional}, recurrent CNN (RCNN) \citep{liang2015recurrent}, and temporal convolutional networks (TCN)\citep{bai2018empirical}.

Another popular thread is to incorporate a tensor decomposition procedure in the model training to reduce computational complexity and potentially enhance interpretability.
Recent examples include tensorizing neural networks \citep{novikov2015tensorizing},  convolutional tensor-train LSTM (Conv-TT-LSTM) \citep{su2020convolutional}, tensor regression networks \citep{kossaifi2020tensor} and its efficient variant using tensor dropout \citep{kolbeinsson2021tensor}. Methods may vary depending on how the tensor decomposition is conducted {\cite{fang2022bayesian,wang2022nonparametric,tao2024undirected}.}

\section{Preliminaries}
\label{sec:prelim}

\noindent \textbf{Notation.} Before proceeding, we first set up some notation. Let $\mathcal{X}\in \mathbb{R}^{d_1 \times \dots \times d_K}$ denote a $K$-th order tensor of dimension $d_1\times d_2 \times \dots \times d_K$. Let $\mathcal{X}_{i_1,\dots,i_K}$ denote the $({i_1,\dots,i_K})$-th entry of the tensor $\mathcal{X}$ and $\mathcal{X}[\mathcal{I}_1,\dots,\mathcal{I}_K]$ denotes the sub-tensor of $\mathcal{X}$ for $\mathcal{I}_1\subseteq \{1,\dots,d_1\},\dots,\mathcal{I}_K\subseteq \{1,\dots,d_K\}$. 
The matricization operation $\mat_k(\cdot)$ unfolds an order-$K$ tensor along mode $k$ to a matrix, say $\cX\in \R^{d_1 \times \dots \times d_K}$ to $\mat_k(\cX)\in\R^{d_k \times d_{-k}}$ where $d_k=\prod_{j\neq k} d_j$ and its detailed definition is provided in the appendix. The Frobenius norm of tensor $\cX$ is defined as $\|\cX\|_{\rm F}=(\sum_{i_1,i_2,\cdots,i_K}\cX_{i_1,\dots,i_K}^2)^{1/2}$. The Tucker rank of
an order-$K$ tensor $\cX$, denoted by Tucrank$(\cX)$, is defined as a $K$-tuple $(r_1,...,r_K)$, where $r_k=\text{rank}(\mat_k(\cX))$. Any Tucker rank-$(r_1,...,r_K)$ tensor $\cX$ admits the following Tucker decomposition \citep{tucker1966some}: $\cX=\cS\times_1 U_1\times_2 \cdots \times_K U_K$, where $\cS\in\R^{r_1\times\cdots \times r_K}$ is the core tensor and $U_k=\text{SVD}_{r_k}(\mat_k(\cX)) $ is the mode-$k$ top $r_k$ left singular vectors. Here, the mode-$k$ product of $\cX\in\R^{d_1 \times \cdots \times d_K}$ with a matrix $ B \in\R^{d_k\times r_k}$, denoted by $\cX\times_k B$, is a $d_1\times \cdots\times d_{k-1}\times r_k \times d_{k+1}\times \cdots \times d_K$-dimensional tensor, and its detailed definition is provided in Appendix A. 
The following abbreviations are used to denote the tensor-matrix product along multiple modes: $\cX\times_{k=1}^K U_k=\cX\times_1 U_1\times_2\cdots\times_K U_K$ and $\cX\times_{\ell\neq k} U_{\ell}=\cX\times_1 U_1\times_2\cdots\times_{k-1} U_{k-1} \times_{k+1} U_{k+1}\times_{k+2}\cdots \times_K U_K$.
For any two matrices $A\in\RR^{m_1\times r_1},B\in \RR^{m_2\times r_2}$, we denote the Kronecker product $\odot$ as $A\odot B\in \RR^{m_1 m_2 \times r_1 r_2}$. For any two tensors $\cA\in\RR^{m_1\times m_2\times \cdots \times m_K}, \cB\in \RR^{r_1\times r_2\times \cdots \times r_N}$, we denote the tensor product $\otimes$ as $\cA\otimes \cB\in \RR^{m_1\times \cdots \times m_K \times r_1\times \cdots \times r_N}$, such that
$(\cA\otimes\cB)_{i_1,...,i_K,j_1,...,j_N}=(\cA)_{i_1,...,i_K}(\cB)_{j_1,...,j_N} .$

\noindent \textbf{Problem Setup.} Suppose we observe a time series dataset comprising $n$ temporal samples $\cD = \{ (\cX_t,\cY_t), 1\le t \le n \} $, where $\mathcal{X}_t \in \mathbb{R}^{d_1 \times \dots \times d_K}$ is a tensor of covariates and $\mathcal{Y}_t \in \mathbb{R}^{p_1 \times \dots \times p_q}$ is the tensor response variable. For some newly observed covariates $\{ \cX_t \}_{t=n+1}^{n+m}$, \emph{the forecasting task of tensor-on-tensor time series regression is to accurately predict the corresponding future tensor responses $\{\cY_t\}_{t=n+1}^{n+m}$, given the observed time series $\{ (\cX_t,\cY_t) \}_{t=1}^{n} $ and the new covariates $\{ \cX_t \}_{t=n+1}^{n+m}$}. Here, $m$ is the forecast horizon. In subsequent discussions, we denote the series in training time range as $\mathcal{X}^{(tr)}=\{ \mathcal{X}_t \}_{t=1}^{n}$ and $\mathcal{Y}^{(tr)}=\{ \mathcal{Y}_t \}_{t=1}^{n}$, and the covariates in forecasting time range as $\mathcal{X}^{(te)}=\{ \mathcal{X}_t \}_{t=n+1}^{n+m}$.

\section{Methodology: Factor Augmented Tensor-on-Tensor Time Series Regression}
\label{sec:methodology}

\begin{figure*}[t]
    \centering
    \includegraphics[width=0.75\textwidth]{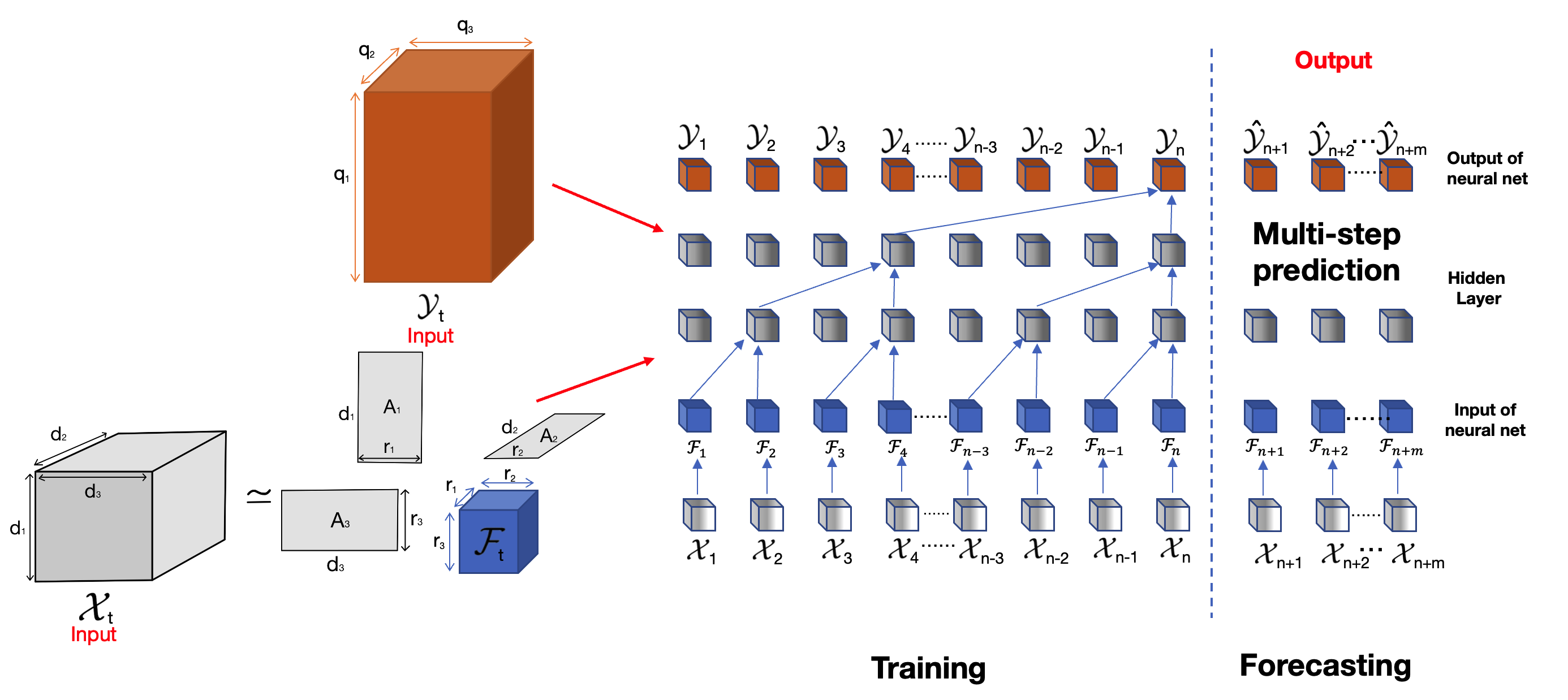} 
     \captionsetup{justification=centering} 
    \caption{A graphical illustration of the proposed FATTNN. The input is the observed time series $\{ (\cX_t,\cY_t) \}_{t=1}^{n} $ and the new covariates $\{ \cX_t \}_{t=n+1}^{n+m}$. The output is the forecasted future tensor responses, denoted by $\{\widehat\cY_t\}_{t=n+1}^{n+m}$ }
\label{fattnn}
\end{figure*}

In this section, we introduce our Factor Augmented Tensor-on-Tensor Neural Network, dubbed as FATTNN.
Our work explores a broad category of challenges known as tensor-on-tensor regression, which seeks to elucidate the relationships between covariates and responses that may manifest as scalars, vectors, matrices, or higher-order tensors. As illustrated in Figure \ref{fattnn}, our method consists of two components. First, factor tensor features are extracted by low-rank tensor factorization. Second, we develop a Tensor-on-Tensor Neural Network based on Temporal Convolutional Network (TCN) \citep{bai2018empirical}. In the first subsection, we present a tensor factor model. This model can capture global patterns among the observed time series covariate tensors, by representing each of the original covariate tensor $\cX_t$ as a multilinear combination of $r_1\times\cdots \times r_K$ basis across each tensor mode $k$ plus noise, where $r_k\ll d_k$. It offers a more general approach than the traditional vector factor analysis for prediction \citep{sen2019think,fan2017sufficient,luo2022inverse,yu2022nonparametric,fan2023factor}, and preserves the tensor structures of features. It also provides dimension reduction, which significantly reduces computational complexity.
In the second subsection, we describe how the output from the global tensor factor model can be used as an input covariate for a TCN to predict future response tensors.

\subsection{Tensor Factor Model (TFM)} \label{sec:factor}

Our representation learning module utilizes low rank tensor factorization for covariate tensor. The idea is to factorize the training covariate tensors $\cX_t\in\R^{d_1\times \cdots\times d_K}, 1\le t\le n,$ into low dimensional tensors $\cF_t\in\R^{r_1\times \cdots \times r_K}$ and linear combination matrices (loading matrices) $A_k\in\R^{d_k\times r_k}$, where $r_k\ll d_k$. Specifically, the model can be expressed as

\begin{align}\label{tfm}
\mathcal{X}_t = \mathcal{F}_t \times_1 A_1 \times_2 \ldots \times_K A_K  +\cE_t,
\end{align}
where $\cE_t$ is idiosyncratic noise of the tensor $\cX_t$, and $t$ is a time point.
Here the \( k \)-mode product of \( \mathcal{X} \in \mathbb{R}^{d_1 \times d_2 \times \ldots \times d_K} \) with a matrix \( U \in \mathbb{R}^{d_k' \times d_k} \), denoted as \( \mathcal{X} \times_k U \), is an order \( K \)-tensor of size \( d_1 \times \ldots \times d_{k-1} \times d_k' \times d_{k+1} \times \ldots \times d_K \) such that
$ (\mathcal{X} \times_k U)_{i_1,\ldots,i_{k-1},j,i_{k+1},\ldots,i_K} = \sum_{i_k=1}^{d_k} \mathcal{X}_{i_1,i_2,\ldots,i_K} U_{j,i_k}.$
If $K=2$, the matrix factor model can be written as $\cX_t=A_1\cF_t A_2^\top +\cE_t$.
The latent core tensor $\cF_t$, which typically encapsulates the most critical information, generally possesses dimensions significantly smaller than those of $\cX_t$. Thus, the covariate tensor $\cX_t$ can be thought of to be comprised of a basis tensor features $\cF_t$ that capture the global patterns in the whole data-set, and all the original covariates are multilinear combinations of these basis tensor features. 
In the next subsection, 
we will discuss how a TCN can be utilized to leverage the temporal structure of the training data.

For predicting future response tensors, given a new covariate tensor $\cX_t$, we can also extract low-dimensional feature tensors $\cF_t$ using the estimated deterministic loading matrices $A_k$ and model \eqref{tfm}. This approach efficiently captures the essential characteristics from the high-dimensional input data. The reduced dimensionality allows for more efficient data processing and analysis while retaining essential features.

To capture the underlying tensor features of the covariates, we employ 
a Time series Inner-Product Unfolding Procedure (TIPUP). It utilizes the linear temporal dependence among the covariates.
The TIPUP method performs singular value decomposition (SVD) for each tensor mode $k=1,...,K$:
\begin{equation}
\widehat A_k=\text{LSVD}_{r_k}  \left( \frac{1}{n} \sum_{t=1}^{n} \mat_k(\cX_{t}) \mat_k^\top(\cX_t)\right).
\label{utipup}
\end{equation}
The estimation methods in \eqref{utipup} can be further enhanced through an iterative refinement algorithm, detailed in the appendix. The factor tensors in the training data thus can be estimated by $\widehat\cF_{t} = \cX_{t} \times_{k=1}^K \widehat A_k^\top,\text{ for }1\le t\le n$.

In the next proposition, we demonstrate the provable benefits of using a tensor factor model with Tucker low-rank structure \eqref{tfm}. Specifically, we establish the convergence rates of the linear space of the loading matrices, which in turn ensure the accuracy of the estimated factor tensors.

\begin{proposition}\label{thm:itipup}
Assume each elements of the idiosyncratic noise $\cE_t$ in \eqref{tfm} are i.i.d. $N(0,1)$. The ranks $r_1,...,r_K$ are fixed. The factor process $\cF_t$ is weakly stationary and its cross-outer-product process is ergodic in the sense of 
$\frac{1}{n}\sum_{t=1}^n \cF_{t}\otimes\cF_t \rightarrow \E (\cF_{t}\otimes\cF_t)$ 
in probability, 
where the elements of $\E (\cF_{t}\otimes\cF_t)$ are all finite. In addition, the condition numbers of $A_k^\top A_k$ ($k=1,...,K$) are bounded. Let $\lambda=\prod_{k=1}^K \| A_k \|_2$ and grow as dimensions $d_k$ increases. Then, the iterative TIPUP estimator satisfies
\begin{align*}
& \|\widehat A_k (\widehat A_k^\top \widehat A_k)^{-1} \widehat A_k^\top -A_k (A_k^\top A_k)^{-1} A_k\|_2 \\
= & O_{\P} \left( \frac{\max_k \sqrt{d_k} }{\lambda \sqrt{n}}+ \frac{\max_k \sqrt{d_k}}{\lambda^2 \sqrt{n}} \right), \qquad 1\le k\le K.   
\end{align*}
\end{proposition}

\subsection{Tensor-on-Tensor Neural Network based TCN (TTNN)} \label{sec:reg}

If we are equipped with a TCN that identifies the temporal patterns in the training dataset $(\cX^{(tr)},\cY^{(tr)})$, we can use this model to enhance the temporal structures in $(\cF^{(tr)},\cY^{(tr)})$. Thus, a TCN is applied to the input sequence of estimated factor tensors $\widehat\cF_1,\ldots, \widehat\cF_{t}$ along with the corresponding output sequence of response tensors $\cY_1,\ldots,\cY_t$ to encapsulate the temporal dynamics. Let $\eta(\cdot)$ represent this network. 
The whole temporal network $\eta(\cdot)$ can be trained using the low-rank factor tensors $\widehat\cF_1,\ldots, \widehat\cF_{n}$ derived from training dataset. 
In the literature, factor-augmented linear regression has been well studied recently in statistics, economics and machine learning \cite{stock2002forecasting, bai2006confidence, bair2006prediction, fan2017sufficient, bing2021prediction}.

The trained temporal network $\eta(\cdot)$ can be effectively utilized for multi-step look-ahead prediction in a standard approach. The factor tensors for future time steps can be estimated using $\widehat\cF_{n+h} = \cX_{n+h} \times_{k=1}^K \widehat A_k^\top$,
where $h>0$. Using the historical data points of factor tensors $\widehat\cF_{t}, 1\le t\le n+h-1$, the prediction for the subsequent time step, $\cY_{n+h}$, is generated by $\widehat\cY_{n+h}=\eta(\widehat\cF_1,...,\widehat\cF_{n+h})$.

Our hybrid forecasting model incorporates a TCN that inputs past data points from the original raw response time series and the estimated factor tensors, enhancing prediction accuracy. The pseudo-code for our model is outlined in Algorithm \ref{alg:tensor_forecast}. Due to space limitations, the determination of the rank of the factor process is discussed in Appendix A. 

\begin{algorithm}
\caption{Factor Augmented Tensor-on-Tensor Neural Network (FATTNN)}
\label{alg:tensor_forecast}
\begin{algorithmic}[1]
\STATE \textbf{Input:} Observed tensor time series \( (\mathcal{X}_1,\mathcal{Y}_1), \ldots, (\mathcal{X}_{n},\mathcal{Y}_{n}), \mathcal{X}_{n+1}, \ldots, \mathcal{X}_{n+m} \), rank \( r_k \) for each mode \( k \), forecasting horizon \( m \).
\STATE \textbf{Output:} Forecasts \( \widehat{\mathcal{Y}}_{n+1}, \ldots, \widehat{\mathcal{Y}}_{n+m} \).
\STATE Compute the factor tensor $\widehat\cF_t$ for $t = 1, \ldots, n$, by 
TIPUP \eqref{utipup} or its iterative version in Algorithm \ref{alg:itipup}.
\STATE Fit a TCN $\eta(\cdot)$ to the input sequence of factor tensors $\widehat\cF_1,\ldots, \widehat\cF_{n}$ and the output sequence of response tensor $\cY_1,...,\cY_{n}$ to capture the temporal dynamics.
\STATE Forecast $\cY_{n+1}, \ldots, \cY_{n+m}$ using the fitted TCN.
\FOR{\( h = 1 \) to \( m \)}
    \STATE $\widehat\cF_{n+h} = \cX_{n+h} \times_{k=1}^K \widehat A_k^\top$.
    \STATE $\widehat{\mathcal{Y}}_{n+h}=\eta(\widehat\cF_1,\ldots,\widehat{\mathcal{F}}_{n+h})$.
\ENDFOR
\RETURN $\widehat{\mathcal{Y}}_{n+1}, \ldots, \widehat{\mathcal{Y}}_{n+m}$.
\end{algorithmic}
\end{algorithm}

\section{Numerical Results}
\label{sec:numericalStudies}

In this section, we investigate the finite-sample performance of our proposed FATTNN methods via simulation studies and real-world applications.
Specifically, 
we implement the proposed \textbf{FATTNN}\footnote{Our code is available on Github: \url{https://github.com/chenezandie/aaai2025-fattnn}.} 
, we consider one traditional statistical model - the multiway tensor-on-tensor regression \citep{lock2018tensor} (denoted by \textbf{Multiway}),
and four state-of-the-art deep learning approaches, including the temporal convolutional network \citep{bai2018empirical} of $\mathcal{Y}$ regressed on $\mathcal{X}$ (denoted by \textbf{TCN}), the long short-term memory network \citep{hochreiter1997long} (denoted by \textbf{LSTM}) , the convolutional tensor-train LSTM \citep{su2020convolutional} (denoted by \textbf{Conv-TT-LSTM}), and tensor regression layer \citep{kossaifi2020tensor} (denoted by \textbf{TRL}).

\smallskip
\noindent \textbf{Evaluation Metrics.} We evaluate the performance of various methods with emphasis on two aspects: prediction accuracy and computational efficiency. 
To evaluate prediction accuracy, we compute the Mean Squared Error (MSE) over the testing data, i.e., $\text{MSE} = (n_{\text{test}}\cdot p_1\cdots p_q)^{-1}\sum_{i\in\mathcal{D}_{\text{test}}} \| \mathcal{Y}_i^{(\text{obs})} - \mathcal{Y}_i^{(\text{pred})} \|_F^2$, 
where $\mathcal{D}_{\text{test}}$ denotes the testing set, $n_{\text{test}}$ is the number of samples in the testing set $\mathcal{D}_{\text{test}}$,  and \( (\mathcal{Y}_i^{(\text{obs})}, \mathcal{Y}_i^{(\text{pred})} ) \) are the observed and predicted values of the $i$-th tensor response.
In addition, we record the computational time of different methods to evaluate computational efficiency.

\subsection{Simulation Studies} \label{subsec:simulation}

\noindent\textbf{Data Generating Process.} We carry out simulation studies under different scenarios with various configurations of sample size, data dimensions, tensor ranks, and relationships between covariates and responses. Recall that $\cX_t\in\R^{d_1\times \cdots\times d_K}$, $\cF_t\in\R^{r_1\times \cdots\times r_K}$, $\cY_t\in\R^{p_1\times \cdots\times p_q}$, for $1\le t\le n$. We begin with generating the low-rank core tensor $\cF_t$ from a tensor autoregressive model ${\rm vec}(\cF_t) = \Phi \cdot {\rm vec}(\cF_{t-1}) + {\rm vec}(\cW_t)$ for $t=1,\ldots,n$. Here, ${\rm vec}(\cF_t)$ denotes the vectorization\footnote{We'd like to clarify that our proposed FATTNN approach does not involve tensor vectorization in the model fitting. This vectorization is solely for the convenience of data generation in simulation.} of a tensor $\cF_t$, $\cW_t \in \R^{r_1\times \cdots\times r_K}$ is an error tensor in which every element is randomly generated from a normal distribution. The coefficient matrix $\Phi \in \R^{ (r_1 r_2\cdots r_K) \times (r_1 r_2\cdots r_K) }$ is constructed using the Kronecker product of a sequence of matrices $Q_k$, i.e., $\Phi = \odot_{k=1}^{K} Q_k$, in which each $Q_k \in \R^{r_k \times r_k}$ is a matrix with i.i.d. standard normal entries and orthonormalized through QR decomposition.  The tensor time series $\cF_t$ is randomly initiated in a distance past, and we discard the first 500 time points to stabilize the process. 

Next, we generate the tensor covariates $\cX_t$ from $\cF_t$,  following
$\mathcal{X}_t = \lambda \mathcal{F}_t \times_1 A_1 \times_2 \ldots \times_K A_K  +\cE_t,$ for $t=1, \ldots n$,
where $\lambda$ is a scalar that controls the signal-to-noise ratio in the tensor factor model. In our experiments, we set $\lambda = (\prod_{k=1}^K r_k)^{1/2}$. The loading matrices $A_k\in\R^{d_k\times r_k}$ are generated with independent $N(0,1)$ entries, and then orthonormalized using QR decomposition. The error $\cE_t\in\R^{d_1\times \cdots\times d_K}$ is also generated with each element independently drawn from normal distributions. 

Then, we generate the tensor responses $\cY_t$ from $\cF_t$, using $\cY_t = \langle S(\cF_t),  \Lambda \rangle_L + \cU_t$, $t=1, \ldots n$. Here $S(\cF_t) := (s(\cF_{t,i_1,\ldots,i_K}))_{1\leq i_1 \leq r_1, \ldots, 1 \leq i_K \leq r_K} \in \R^{r_1\times \cdots\times r_K}$ applies an element-wise transformation $s(\cdot)$ to each entry of $\cF_t$. $\cU_t \in \R^{p_1\times \cdots\times p_q}$ is a noise matrix whose elements are i.i.d $N(0, \sigma_u^2)$. The coefficient tensor $\Lambda$ is set as $\Lambda =\llbracket U_1,\cdots,U_K, V_1,\cdots V_d \rrbracket \in \R^{r_1 \times \cdots \times r_K \times p_1 \times \cdots \times p_q}$ where $ U_k \in \R^{r_k \times R} $ for $ k = 1,\cdots , K $ and $ V_m \in\R^{p_m \times R} $ for $ m = 1, \cdots , d$, each with independent $N(0, 1)$ entries. In our experiments, we set $R=6$.
Details of the notation $\llbracket \ \cdot, \cdots, \cdot \ \rrbracket$ and $\langle \cdot, \cdot \rangle_L$ are provided in Equations \eqref{eq:append-notation1} and \eqref{eq:append-notation2}.

We consider the following experimental configurations: \\
(1) {\small $(d_1, d_2,  d_3)=(25, 25, 12)$,$(r_1, r_2,  r_3) = (3, 3, 2)$, $(p_1, p_2,  p_3) = (6, 8, 6)$, $n=500$, $s(z)=cos(z)$, $\sigma_u^2 = 1$.} \\
(2) {\small $(d_1, d_2,  d_3) = (30, 6, 12)$, $(r_1, r_2,  r_3) = (6, 3, 2)$, $(p_1, p_2,  p_3) = (8, 6, 4)$, $n=400$, $s(z)=log(|z|)$, $\sigma_u^2 = 1$.} \\ 
(3) {\small $(d_1, d_2,  d_3) = (12, 3, 12)$, $(r_1, r_2,  r_3) = (4, 3, 4)$,  $(p_1, p_2,  p_3) = (3, 3, 3)$, $n=100$, $s(z)=log(e^{(z)}+1)$, $\sigma_u^2 = 0.5$.}

\smallskip
\noindent\textbf{Simulation Evaluations.} In each setting, we use 70\% for training the models and 30\% for model evaluation. Comparisons of MSE, and computational time are summarized in Table \ref{tab:simulation}. 
All reported metrics are averaged over 20 replications. Experiments were run on the research computing cluster provided by the University Center for Research Computing. From Table 1, we can see that FATTNN consistently improves over the rest benchmark methods by a large margin. The improvements can be attributed to two facets: (a) the utilization of tensor factors, summarizing useful predictive information effectively and reducing the number of noise covariates, and (b) the utilization of neural networks, possessing extraordinary ability in capturing complex relationships in the data and therefore exhibiting exceptional predictive power. 
On a separate note, Table \ref{tab:simulation}
reveals the evident computational advantage of FATTNN compared to other neural network models such as TCN, owing to the effective dimension reduction through tensor factor models.

\begin{table}
    \centering
    \captionsetup{justification=centering} 
    \caption{\footnotesize Comparisons of prediction accuracy (measured by MSE over the testing data) and computational time (in the format of hh:mm:ss) in simulation studies}  \label{tab:simulation}
    \footnotesize
    \begin{tabular}{c|c|c|c}    
    \hline
    Prediction Task & Simulation 1 & Simulation 2 & Simulation 3
     \\ \hline
    & \multicolumn{3}{c}{Test MSE} \\ 
    TCN & 9.812 (3.13) & 670.5 (25.8) & 1074  (32.6)\\
    FATTNN & \textbf{6.353} (2.52) & \textbf{511.9} (22.5) & \textbf{696.7} (26.4)\\
    MultiWay & 12.31 (3.51) &805.1 (28.4)  & 1670  (40.3) \\
    LSTM & 10.48 (3.23) & 878.3 (29.6)  & 807.8 (28.3) \\
    TRL& 12.29 (3.50) & 852.6 (29.0) & 2003 (38.4) \\
     Conv-TT-LSTM &10.32 (3.15) & 1643 (39.5)& 2812 (52.8)\\
    \hline
    & \multicolumn{3}{c}{Computational Time} \\ 
    TCN & 01:22:02& 01:26:30& 00:11:16\\
    FATTNN & 00:16:18& 00:17:54& 00:03:41\\
    MultiWay & 00:20:15&00:12:30& 00:01:03\\
    LSTM & 00:03:36& 00:03:01& 00:02:09\\
    TRL& 00:02:25& 00:01:30& 00:00:50\\
    Conv-TT-LSTM &03:02:15&05:31:25& 00:35:37\\
    \hline
    \end{tabular}

    Note: The results are averaged over 20 replications. Numbers in the brackets are standard errors. 
\label{tab:simulationAccuracy}
\end{table}
\subsection{Real Data Examples}
We evaluate the performance of different methods using five prediction tasks on three real-world datasets. 
The information about sample sizes in the datasets, data dimensions of covariates $(d_1, \ldots, d_K)$ and responses $(p_1, \ldots, p_q)$, and the ranks of core tensors $(r_1, \ldots, r_K)$ utilized in the analyses are summarized in Table \ref{tab:realdateinfo}. 
For the prediction tasks using the New York Taxi data and FMRI image datasets, we split the data into training sets (70\%) and testing (30\%). For the prediction tasks using the FAO dataset, we use 80\% data for training and 20\% for testing due to its relatively small sample size. We compare the prediction accuracy and computation time of six methods illustrated in the simulation section. The numerical prediction results along with the corresponding computational time are summarized in Table \ref{tab:realDataAccuracy}.

\begin{table*}[!htb]
    \centering
    \caption{\footnotesize Data information of the real-world applications}
    \footnotesize
    \begin{tabular}{c|c|c|c|c}
    \hline
    Prediction Task & Sample Size & Dimension of $\mathcal{X}_t$ & Dimension of $\mathcal{Y}_t$ & Rank of $\cF_t$
     \\ \hline
    FAO crop $\sim$ crop      &    62 &   (33,2,11) &   (13,2,11) &   (6,2,4) \\
    FAO crop $\sim$ livestock &    62 &    (26,4,5) &   (26,3,11) &   (6,4,2) \\
    NYC taxi -- midtown       &   504 & (2,12,12,8) & (12,12,8)& (2,4,4,2) \\
    NYC taxi -- downtown      &   504 & (2,19,19,8) & (19,19,8)& (2,4,4,2) \\
    FMRI data           &  1452 &  (25,64,64) &   (1,64,64) & (8,23,23) \\
    \hline
    \end{tabular}
\label{tab:realdateinfo}
\end{table*}

\begin{table*}
\centering
\captionsetup{justification=centering} 
    \caption{\footnotesize Comparisons of prediction accuracy (measured by MSE over the testing data) and computational time \\ (in the format of hh:mm:ss) of different methods 
    in real-world applications }
    \small
    \resizebox{\textwidth}{!}{
    \begin{tabular}{c|cccccccccccccc}
    \hline
    Prediction Task
    & TCN  & FATTNN & Multiway & TRL & LSTM& Conv-TT-LSTM
     \\ \hline
    
    &  \multicolumn{6}{c}{Test MSE} \\
    FAO crop $\sim$ crop      &  14.11 [10.83, 17.58]  &\textbf{7.437} [6.18, 11.44]  & 12.29 [11.86, 32.45] & 17.84 [15.89, 22.74] & 15.76 [13.16, 18.34] &53.39 [52.41, 63.95] \\ 
    FAO  crop$\sim$ livestock &  32.53 [24.87, 36.10]  &\textbf{21.65} [19.56, 25.10] & 48.04  [45.56, 165.80]  & 66.25 [56.85, 76.13] & 34.23 [30.47, 38.48] & 69.61 [69.06, 78.44] \\ 
    NYC taxi -- midtown       &  12.87 [12.06, 15.52] &\textbf{7.399} [6.53, 8.03] & 22.84 [20.86, 44.79]  & 8.565 [7.89, 10.12] & 8.404 [7.95, 9.63] & 25.80 [21.55, 28.62] \\ 
    NYC taxi -- downtown      &  14.58 [13.36, 17.42]  &\textbf{9.493} [8.46, 11.03] & 28.10 [23.99, 52.12] & 12.43 [10.82, 14.50] & 12.45 [11.09, 14.38] & 42.18 [36.27, 47.66] \\ 
    FMRI data                 & 0.1236 [0.0997, 0.148]  & 0.06634 [0.053, 0.075] & 0.1397 [0.135, 0.158] &0.07902 [0.072, 0.084] & 0.1964  [0.163, 0.224] &\textbf{0.03871} [0.0337, 0.0914]\\
    \hline
    &  \multicolumn{6}{c}{Computational Time} \\ 
    FAO crop $\sim$ crop       & 00:55:33 & 00:15:28 & 00:03:18 & 00:09:41 & 00:02:08&00:20:07 \\ 
    FAO  crop$\sim$ livestock  & 00:24:34 & 00:12:57 & 00:03:21 & 00:07:53 & 00:01:55&00:16:18 \\ 
    NYC taxi -- midtown        & 01:57:15 & 00:34:34 & 00:04:52 & 00:19:26 & 00:05:01&01:54:28 \\ 
    NYC taxi -- downtown       & 03:25:55 & 00:32:46 & 00:25:35 & 00:22:09 & 00:05:34&03:07:05 \\ 
    FMRI data                  & 03:54:50 & 00:20:58 & 10:01:20 & 00:43:51 & 00:06:27&17:27:05\\
    \hline
    \end{tabular}
    }

    Note: Numbers in the square brackets are bootstrap confidence intervals (CIs) of test MSE over 100 bootstrapping replications. 
\label{tab:realDataAccuracy}
\end{table*}

\noindent\textbf{(1) The United Nations Food and Agriculture Organization (FAO) Crops and Livestock Products Data.}
The database provides agricultural statistics (including crop, livestock, and forestry sub-sectors) collected from countries and territories since 1961. It is publicly available at \url{https://www.fao.org}. 
Our analyses focus on 11 crops and 5 livestock products from 46 countries in East Asia, North America, South America, and Europe from 1961 to 2022. Detailed information about the types of crops, livestock, and countries are presented in Table \ref{tab:FAO-statistics} of Appendix B. 
We study two prediction tasks:
(i) using Yield and Production data of 11 different crops 
for 33 countries in East Asia, North America, and Europe (i.e., $\cX_t \in \mathbb{R}^{33\times2\times11}$) to predict the Yield and Production quantities of the same crops for 13 selected countries in South America (i.e., $\cY_t \in \mathbb{R}^{13\times2\times11}$); 
and (ii) using four agricultural statistics (Producing Animals, Animals-slaughtered, Milk Animals, Laying) associated with 5 kinds of livestock of 26 selected countries in Europe (i.e., $\cX_t \in \mathbb{R}^{26\times4\times5}$) to predict three metrics (Area-harvested, Production, and Yield) of 11 crops in the same countries (i.e., $\cY_t \in \mathbb{R}^{26\times3\times11}$). 
Observing high variability in the raw data, we adopt a log transformation on the raw series and fit our models using log-transformed data.

\begin{figure*}[tb]
\centering
\includegraphics[width=0.45\textwidth, height = 1.7in]{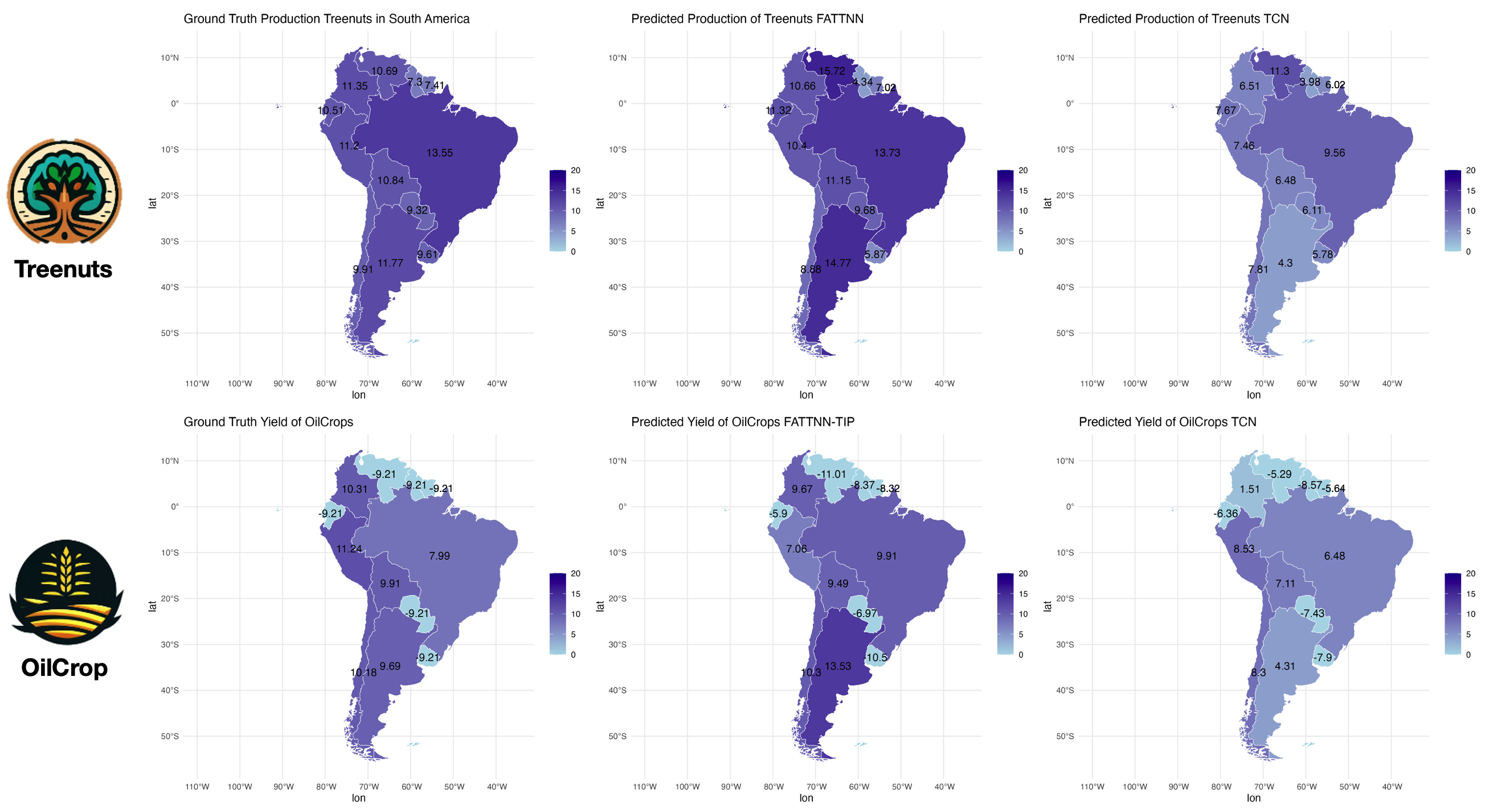} \quad
    \includegraphics[width=0.45\textwidth, height = 1.7in]{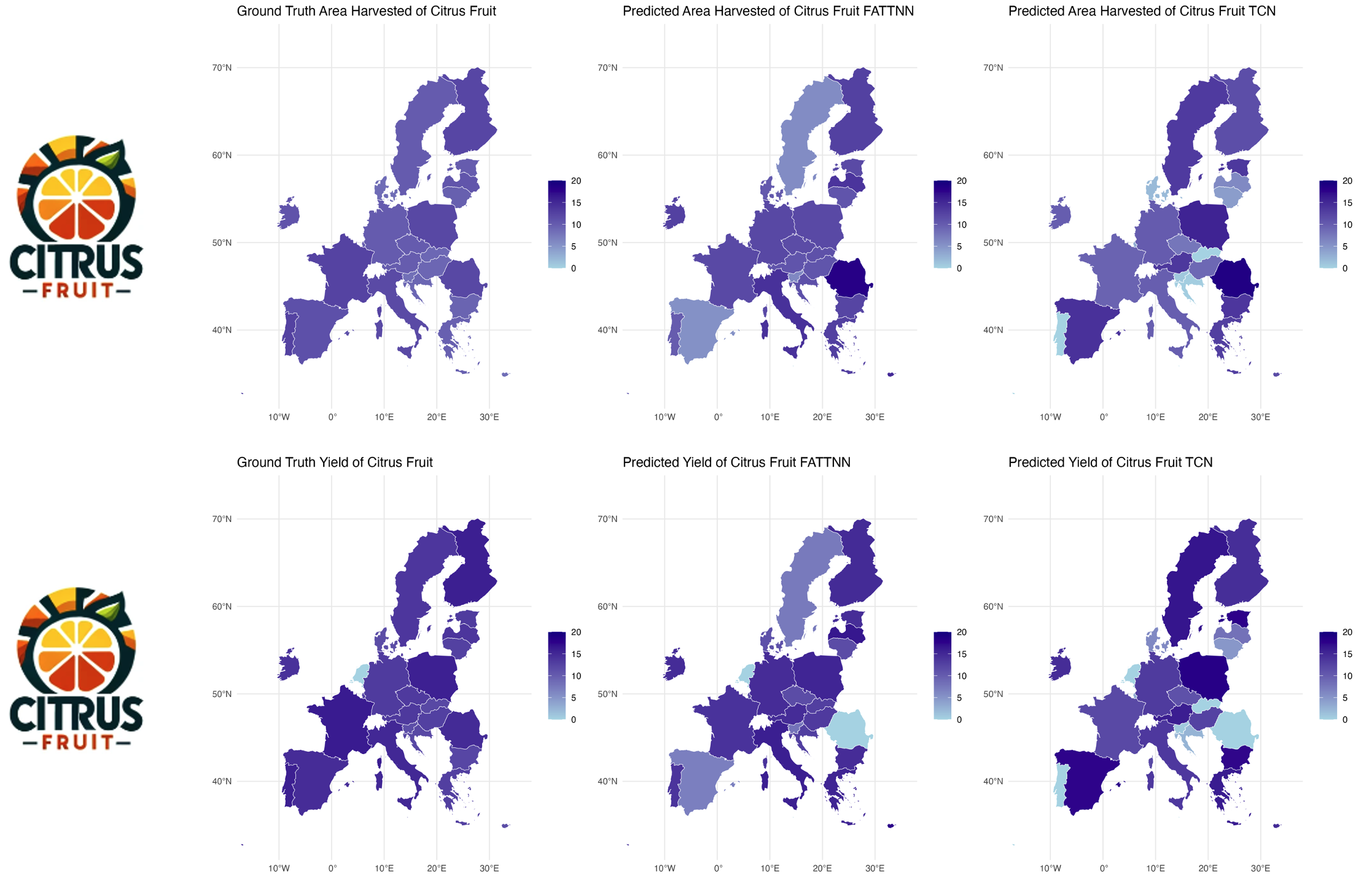} 
    \caption{\footnotesize Left Panel: Comparisons between ground-truth values and predicted Production of Treenuts (top row) and Yield of Oilcrops (bottom row) in South America using the crop data of 33 countries in East Asia, North America, and Europe; Right Panel: Comparisons between ground-truth values and predicted Area-harvested (top row) and Yield (bottom row) of citrus fruit in 26 selected countries in Europe using livestock data of the same 26 countries. Values are plotted on the log-transformed scale. 
    For readability, numerical values from the figures are tabulated in Appendix B. 
    In each panel, from left to right: Ground truth,  FATTNN, and TCN. } 
    \label{fig:FAO}
\end{figure*}

Table \ref{tab:realDataAccuracy} shows that FATTNN brings in substantial reductions in MSE compared to other methods. In task (i), FATTNN yields \textbf{47.3\%, 39.5\%, 58.3\%, 52.8\%, and 86.1\% reductions in MSE compared to TCN, Multiway, TRL, LSTM, and Conv-TT-LSTM}, respectively, and in task (ii), the improvements are \textbf{33.4\%, 54.9\%, 67.3\%, 36.7\%, and 68.9\% reductions in MSE}, respectively. 
The table also shows that the upper bound of the bootstrap CI for FATTNN is smaller than the lower bounds of the CIs for almost all the benchmarks, indicating FATTNN possesses a statistically significant increase in prediction accuracy. What's more, FATTNN produces the shortest CI, implying the lowest variability among the methods compared. 
As reflected by the table, FATTNN \textbf{computes much faster than TCN, about 3.6 times faster in task (i) and 2 times faster in task (ii)}. We also provide graphical comparisons between ground truth and predictions from different methods; see Figure \ref{fig:FAO}.
The plots confirm that FATTNN yields results that most closely align with the ground-truth patterns.

\noindent\textbf{(2) New York City (NYC) Taxi Trip Data.}
The data contains 24-hour taxi pick-up and drop-off information of 69 areas in New York City for all the business days in 2014 and 2015. It is publicly available at \url{https://www.nyc.gov}. 
Considering the autoregressive nature of this taxi data, we include the lag-1 response as an additional covariate when fitting the models. 
We focus on two prediction tasks: using the pick-up and drop-off data from 6:00-14:00 to predict the pick-up and drop-off outcomes from 14:00-22:00 in 12 districts in Midtown Manhattan and 19 districts in Downtown Manhattan, respectively. A detailed Manhattan district map is provided in Figure \ref{fig:ManhattanMap} of Appendix B. 
The Midtown prediction task yields a tensor covariate and response $(\cX_t, \cY_t) \in \mathbb{R}^{2\times12\times12\times8} \times \mathbb{R}^{12\times12\times8}$ for each business day, in which the 1st dimension stands for the observed information at time $t$, concatenated with lag 1 response, the 2nd and 3rd dimensions encode the pick-up and drop-off districts, and the 4th dimension represents the hour of day. Similarly, the Downtown prediction task involves $(\cX_t, \cY_t) \in \mathbb{R}^{2\times19\times19\times8} \times \mathbb{R}^{19\times19\times8}$for each business day.

From Table \ref{tab:realDataAccuracy}, we observe that FATTNN significantly outperforms all other approaches in terms of prediction accuracy, with \textbf{42.5\%, 67.6\%, 13.6\%, 12.0\%, and 71.3\% reductions in MSE compared to TCN, Multiway, TRL, LSTM, and Conv-TT-LSTM}, respectively, in Midtown traffic prediction. In Downtown traffic prediction, FATTNN shows 34.9\%, 66.2\%, 23.6\%, 23.7\%, and 77.5\% reductions in MSE compared to TCN, Multiway, TRL, LSTM, and Conv-TT-LSTM, respectively. As shown in the table, FATTNN is much more computationally efficient than TCN which directly feeds tensor observations into a convolutional network, about 6.4 times faster than TCN in Downtown prediction and 3.4 times faster in Midtown prediction. Additionally, Figure \ref{fig:taxi} visualizes the comparisons using a 5-day moving average of ground-truth values and predicted pick-up and drop-off volumes by day (where we sum the number of pick-up and drop-off across 8-hour testing periods). 
The FATTNN-predicted lines appear to be closer to the ground-truth lines than TCN-predicted lines in most plots.

\begin{figure}[!htb]
    \centering
    \includegraphics[width=0.42\textwidth]{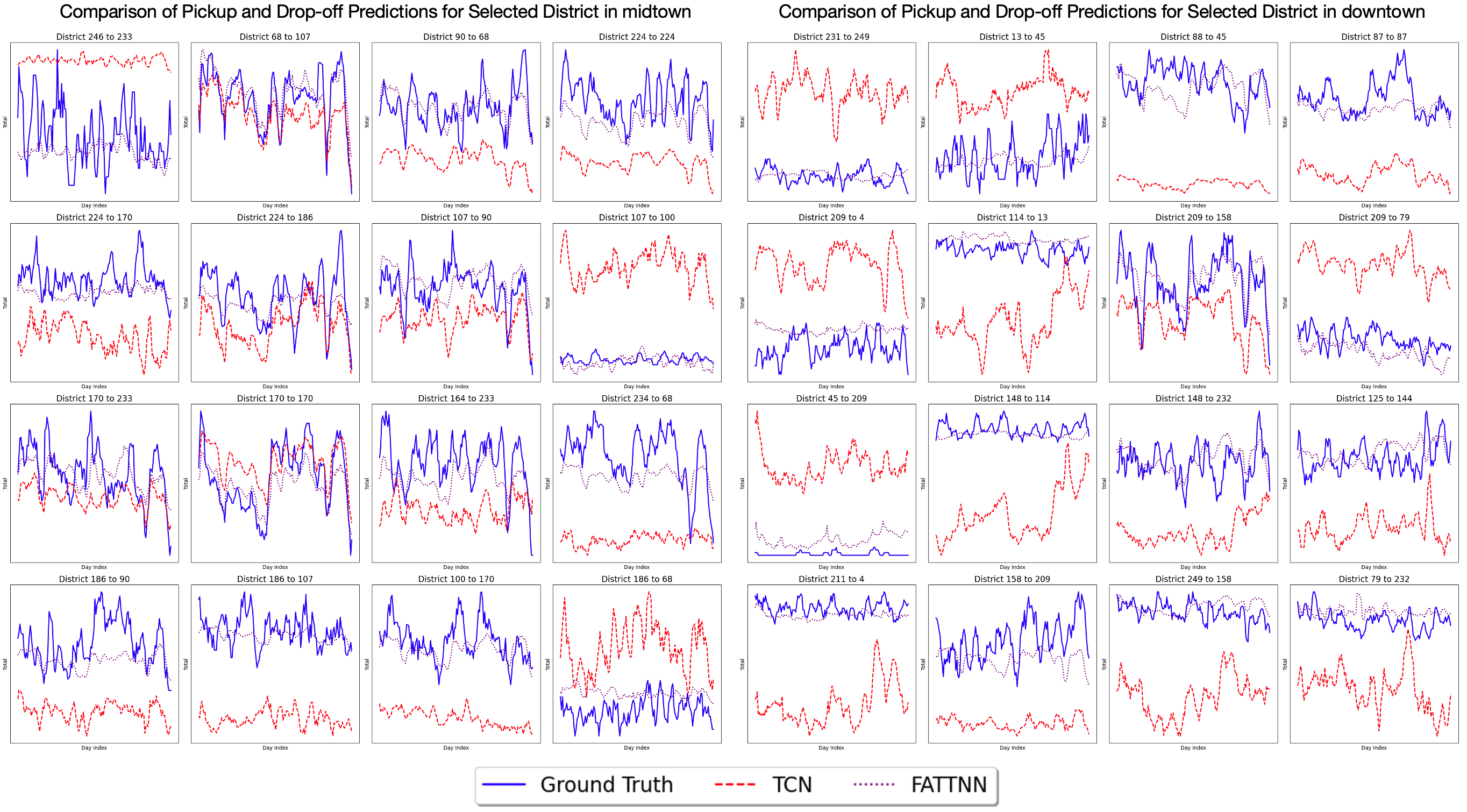} 
    \caption{\footnotesize Comparisons between ground-truth values and predicted pick-up and drop-off volumes using various methods. ``District A to B'' denotes the traffic volume that passengers were picked up in District A and dropped off in District B. The district numbering is assigned according to the Manhattan district map shown in Figure \ref{fig:ManhattanMap}. An enlarged figure and more detailed descriptions are in Appendix B. Left: Midtown Manhattan; Right: Downtown Manhattan. 
    }
\label{fig:taxi}
\end{figure}
\noindent\textbf{(3) Functional Magnetic Resonance Imaging (FMRI) Data.} 
We consider the Haxby dataset \citep{haxby2001distributed}, a well-known public dataset for research in brain imaging and cognitive neuroscience, which can be retrieved from the ``NiLearn'' Python library. 
The data contains slices of $64\times 64$ FMRI brain images collected from 1452 samples. We use the first 25 slices of each sample to predict the 26th slice of that sample, i.e., $(\mathcal{X}_{t}, \mathcal{Y}_{t})\in \mathbb{R}^{25\times64\times64} \times \mathbb{R}^{1\times64\times64}$.

Conv-TT-LSTM yields the lowest MSE in this task. This is not surprising as Conv-TT-LSTM is specifically designed for image prediction while FATTNN is more suitable for time-series tensor prediction. Other than Conv-TT-LSTM, 
FATTNN still outperforms other methods. Additionally, though Conv-TT-LSTM has the best prediction accuracy, it is quite computationally heavy - about 50 times slower than FATTNN. The FATTNN is also 2 times faster than the TRL method which has the 3rd lowest MSE in this task.

\smallskip
\noindent \textbf{In summary}, our proposed FATTNN performs the best in the four prediction tasks on FAO and NYC datasets. In the FMRI dataset, the FATTNN yields higher MSE than Conv-TT-LSTM, a method specialized in image prediction. That being said, FATTNN has a huge computational advantage over Conv-TT-LSTM. We also observe that Conv-TT-LSTM underperforms other methods in the FAO and NYC datasets by a large margin, when the prediction tasks are not focusing on image data. To summarize, the FATTNN demonstrates the overall best performance on all the prediction tasks, and strikes an ideal balance between prediction accuracy and computational cost. The FATTNN achieves substantial increases in prediction accuracy compared to benchmark methods. 
At the same time, the computational costs of FATTNN methods are much lower compared to TCN and Conv-TT-LSTM, ranging from 2 to 50 times faster.

\section{Summary \& Discussion} \label{sec:summary}

In this paper, we propose FATTNN, a hybrid method that integrates a tensor factor model into deep neural networks, for tensor-on-tensor time series forecasting. The key advantage of FATTNN is its ability to exploit and utilize the underlying tensor structure among the covariate tensors.
Through a low-rank tensor factor representation, we preserve the tensor structures and, in the meantime, drastically reduce the data dimensionality, leading to enhanced prediction results and reduced computational costs. This approach unifies the strengths of both factor models and neural networks. 
Numerical studies confirm that our proposed methods successfully \emph{achieve substantial increases in prediction accuracy and significant reductions in computational time}, compared to traditional linear tensor-on-tensor regression models (which often fail to capture the intricate nonlinearity in covariate-response relationships), conventional deep learning methods (that often operate in a black-box manner without making use of the inherent structures among observations), and state-of-the-art tensor-decomposition-based learning methods.

\smallskip
\noindent \textbf{Limitation and Extension.}  
Though we presented our models for the purpose of tensor-on-tensor time series forecasting, \emph{our proposed FATTNN framework is not confined to temporal data but can accommodate a variety of data types}. 
The main idea of FATTNN is to use a low-rank tensor factor model to capture the intrinsic patterns among the observed covariates, and then proceed with a neural network to model the intricate relationships between covariates and responses. When handling time series tensors, we adopt TIPUP factorization and TCN to fully capitalize on their temporal nature.
For non-temporal data, the TIPUP-TCN-based FATTNN is still applicable for tensor-on-tensor prediction tasks, but the strengths of TIPUP and TCN may not be fully utilized. For higher-order tensors (e.g. $K > 4$ ), Tensor-Train \citep{oseledets2011tensor} and Hierarchical Tucker decompositions \citep{lubich2013dynamical} may be better suited.
This limitation can be easily rectified: 
with appropriate choices of factor models and neural network architectures, the proposed FATTNN can be naturally adapted to any simple (e.g., i.i.d observations) or complex tensor-type data (e.g., image, network, text, video). 
\emph{To extend beyond time series data, different types of factor models and neural network architectures should be deployed, depending on the nature of the data.} For example, for graph data, Graph Neural Network is a more suitable alternative to TCN. For text data, Transformer models could substitute for the TCN in our FATTNN framework. An extended discussion on the generality of the FATTNN framework is presented in Appendix C. Our proposed FATTNN is flexible in accommodating different types of tensor factor models and deep learning architectures for tensor-on-tensor prediction using diverse data types.

\section{Acknowledgments}
We thank the anonymous reviewers for helpful suggestions. The work of Yuefeng Han is supported in part by National Science Foundation Grants DMS-2412578.  The work of Xiufan Yu is supported in part by National Institutes of Health grant R01GM152812.


\bibliography{ref-tensor}

\appendix
\setcounter{table}{0}
\setcounter{figure}{0}
\setcounter{algorithm}{0}
\setcounter{equation}{0}
\renewcommand{\thefigure}{S.\arabic{figure}}
\renewcommand{\thetable}{S.\arabic{table}}
\renewcommand{\thealgorithm}{S.\arabic{algorithm}}
\renewcommand{\theequation}{S.\arabic{equation}}

\newpage
\section{Appendix A: Tensor Factor Models} \label{append:tensor-factor-model}
\subsection{Notation}
For a vector $x=(x_1,...,x_p)^\top$, define $\|x\|_q = (x_1^q+...+x_p^q)^{1/q}$, $q\ge 1$. For a matrix $A = (a_{ij})\in \RR^{m\times n}$, write the SVD as $A=U\Sigma V^\top$, where $\Sigma=\text{diag}(\sigma_1(A), \sigma_2(A), ..., \sigma_{\min\{m,n\}}(A))$, with 
singular values $\sigma_{\max}(A) = $ $\sigma_1(A)\ge\sigma_2(A)\ge \cdots\ge \sigma_{\min\{m,n\}}(A)\ge 0$ in descending order. The matrix spectral norm is denoted as
$\|A\|_{2}=\sigma_1(A).$ Write $\sigma_{\min}(A)$ the smallest nontrivial singular value of $A$.
For two sequences of real numbers $\{a_n\}$ and $\{b_n\}$, write $a_n=O(b_n)$ (resp. $a_n\asymp b_n$) if there exists a constant $C$ such that $|a_n|\leq C |b_n|$ (resp. $1/C \leq a_n/b_n\leq C$) for all sufficiently large $n$, and write $a_n=o(b_n)$ if $\lim_{n\to\infty} a_n/b_n =0$. Write $a_n\lesssim b_n$ (resp. $a_n\gtrsim b_n$) if there exist a constant $C$ such that $a_n\le Cb_n$ (resp. $a_n\ge Cb_n$). Write $a\wedge b=\min\{a,b\}$ and $a\vee b=\max\{a,b\}$. We use $C, C_1,c,c_1,...$ to denote generic constants, whose actual values may vary from line to line. 

For any two $m\times r$ matrices with orthonormal columns, say, $U$ and $\widehat U$, suppose the singular values of $U^\top \widehat U$ are $\sigma_1\ge \sigma_2 \ge \cdots \ge \sigma_r\ge 0$.
Define the principal angles between $U$ and $\widehat U$ as
$$ \Theta(U,\widehat U)=\text{diag} (\arccos(\sigma_1),\arccos(\sigma_2),...,\arccos(\sigma_r)). $$
A natural measure of distance between the column spaces of $U$ and $\widehat U$ is then
\begin{equation}\label{loss}
\|\widehat U\widehat U^\top - UU^\top\|_{2}= \|\sin \Theta(U,\widehat U) \|_{2}=\sqrt{1-\sigma_r^2}, 
\end{equation}
which equals to the sine of the largest principal angle
between the column spaces of $U$ and $\widehat U$. 
For any two matrices $A\in\RR^{m_1\times r_1},B\in \RR^{m_2\times r_2}$, denote the Kronecker product $\odot$ as $A\odot B\in \RR^{m_1 m_2 \times r_1 r_2}$. For any two tensors $\cA\in\RR^{m_1\times m_2\times \cdots \times m_K}, \cB\in \RR^{r_1\times r_2\times \cdots \times r_N}$, denote the tensor product $\otimes$ as $\cA\otimes \cB\in \RR^{m_1\times \cdots \times m_K \times r_1\times \cdots \times r_N}$, such that
$$(\cA\otimes\cB)_{i_1,...,i_K,j_1,...,j_N}=(\cA)_{i_1,...,i_K}(\cB)_{j_1,...,j_N} .$$
Let ${\rm{vec}}(\cdot)$ be the vectorization of matrices and tensors. The mode-$k$ unfolding (or matricization) is defined as ${\rm{mat}}_k(\cA)$, which maps a tensor $\cA$ to a matrix 
${\rm{mat}}_k(\cA)\in\RR^{m_k\times m_{-k}}$ where $m_{-k}=\prod_{j\neq k}^K m_j$.  For example, if $\cA\in\RR^{m_1\times m_2\times m_3}$, then
$$({\rm{mat}}_1(\cA))_{i,(j+m_2(k-1))}= ({\rm{mat}}_2(\cA))_{j,(k+m_3(i-1))} $$ $$= ({\rm{mat}}_3(\cA))_{k,(i+m_1(j-1))} =\cA_{ijk}.  $$
For tensor $\cA\in\RR^{m_1\times m_2\times \cdots \times m_K}$, the Hilbert Schmidt norm or Frobenius norm is defined as
$$ \|\cA\|_{{\rm F}}=\sqrt{\sum_{i_1=1}^{m_1}\cdots\sum_{i_K=1}^{m_K}(\cA)_{i_1,...,i_K}^2 }. $$

For vectors $a_1, \ldots, a_K$ of length $r_1, \ldots, r_K$, the outer product
\begin{equation*}
    \cA = a_1 \circ a_2 \circ \cdots \circ a_K \in \mathbb{R}^{r_1 \times \cdots \times r_K}
\end{equation*}
is the $K$-th order tensor, with entries
$\cA[i_1, \ldots, i_K] = \prod_{k=1}^K a_k[i_k]$.

Let $M_1 \in \R^{r_1\times R}, \ldots M_K \in \R^{r_K\times R}$ be $K$ matrices with the same column dimension (denoted by $R$), the notation $\llbracket M_1, \cdots, M_K \rrbracket$ represents 
\begin{equation}\label{eq:append-notation1}
\llbracket M_1, \ldots, M_K \rrbracket = \sum_{r=1}^{R} m_{1r} \circ \cdots \circ m_{Kr},
\end{equation}
where \(m_{kr}\) is the \(r\)th column of \(M_k\), $k=1,\ldots, K$, $r= 1, \ldots, R$.

For two tensor \( \mathcal{A}: d_1 \times \cdots \times d_K \times r_1 \times \cdots \times r_K \) and \( \mathcal{B}: r_1 \times \cdots \times r_K \times p_1 \times \cdots \times p_q \), the notation $\langle \cA, \cB \rangle_L$ denote the contracted tensor product \citep{lock2018tensor}

\begin{equation*}
    \langle \mathcal{A}, \mathcal{B} \rangle_L  \in \R^ {d_1 \times \cdots \times d_K \times p_1 \times \cdots \times p_q}
\end{equation*}
with the $(i_1, \ldots, i_K, j_1, \ldots, j_d)$-th element of $\langle \mathcal{A}, \mathcal{B} \rangle_L$ defined as
\begin{align}\label{eq:append-notation2}
\begin{aligned}
    & \langle \mathcal{A}, \mathcal{B} \rangle_L[i_1, \ldots, i_K, j_1, \ldots, j_d] = \\ & \sum_{l_1=1}^{r_1} \cdots \sum_{l_K=1}^{r_K} \mathcal{A}[i_1, \dots, i_K, l_1, \dots, l_K] \mathcal{B}[l_1, \dots, l_K, j_1, \dots, j_d].
\end{aligned}
\end{align}

\subsection{Preliminary of matrix and tensor algebra}
\noindent \textbf{Fact (Norm inequalities)}. 
\begin{itemize}
\item For any matrix $A, B$, $\|AB\| \geq \|A\| \cdot \sigma_{\min}(B)$.
\item For any matrix $A$, square invertible matrix $B$, $\|A\| \leq \sigma_{\max}(B) \cdot \|AB^{-1}\|$.
\item For any matrix $A$, and square diagonal invertible matrix $B$, $\sigma_{\min}(BA) \geq \sigma_{\min}(B) \cdot \sigma_{\min}(A)$.
\end{itemize}

\begin{lemma}[Weyl's inequality] \label{thm:weyl}
Let $A,B \in \R^{m\times n} $ where $n \geq k$, then we have for any $i \in [k]$,
\begin{equation}
    |\sigma_i(A)-\sigma_i(B)| \leq \|A-B\| .
\end{equation}
\end{lemma}

\begin{lemma}[Davis–Kahan $\sin\Theta$ theorem]\label{thm:sin}
Let $\Sigma, \hat{\Sigma} \in \mathbb{R}^{p \times p}$ be symmetric, with eigenvalues $\lambda_1 \geq \cdots \geq \lambda_p$ and $\hat{\lambda}_1 \geq \cdots \geq \hat{\lambda}_p$ respectively. Fix $1 \leq r \leq s \leq p$, let $d = s - r + 1$, and let $V = (v_r, v_{r+1}, \ldots, v_s) \in \mathbb{R}^{p \times d}$ and $\hat{V} = (\hat{v}_r, \hat{v}_{r+1}, \ldots, \hat{v}_s) \in \mathbb{R}^{p \times d}$ have orthonormal columns satisfying $\Sigma v_j = \lambda_j v_j$ and $\hat{\Sigma} \hat{v}_j = \hat{\lambda}_j \hat{v}_j$ for $j = r, r + 1, \ldots, s$. Write $\delta = \inf\{|\hat{\lambda} - \lambda| : \lambda \in [\lambda_s, \lambda_r], \hat{\lambda} \in (-\infty, \hat{\lambda}_{s+1}] \cup [\hat{\lambda}_{r-1}, \infty)\}$, where we define $\hat{\lambda}_0 = -\infty$ and $\hat{\lambda}_{p+1} = \infty$, and assume that $\delta > 0$. Then
\begin{equation}
    \|\sin \Theta(\hat{V}, V)\| \leq \frac{\|\hat{\Sigma} - \Sigma\|}{\delta},
\end{equation}
where the norm $\|\cdot\|$ could be spectral norm $\|\cdot\|_2$ or Frobenious norm $\|\cdot\|_{\rm F}$. Frequently in applications, we have $r = s = j$, say, in which case we have
\begin{equation}
\|\sin \Theta(\hat{v}_j, v_j)\|_2 \leq \frac{\|\hat{\Sigma} - \Sigma\|_{2}}{\min(|\hat{\lambda}_{j-1} - \lambda_j|, |\hat{\lambda}_{j+1} - \lambda_j|)}.
\end{equation}
\end{lemma}

\subsection{Technical lemmas}
\begin{lemma}[$\epsilon$-covering of matrix norms, \citet{han2020tensor}]\label{lemma:epsilonnet}
Let $d, d_j, d_*, r\le d\wedge d_j$ be positive integers, $\epsilon>0$ and
$N_{d,\epsilon} = \lfloor(1+2/\epsilon)^d\rfloor$. \\
(i) For any norm $\|\cdot\|$ in $\R^d$, there exist
$M_j\in \R^d$ with $\|M_j\|\le 1$, $j=1,\ldots,N_{d,\epsilon}$,
such that $\max_{\|M\|\le 1}\min_{1\le j\le N_{d,\epsilon}}\|M - M_j\|\le \epsilon$.
Consequently, for any linear mapping $f$ and norm $\|\cdot\|_*$,
$$
\sup_{M\in \R^d,\|M\|\le 1}\|f(M)\|_* \le 2\max_{1\le j\le N_{d,1/2}}\|f(M_j)\|_*.
$$
(ii) Given $\epsilon >0$, there exist $U_j\in \R^{d\times r}$
and $V_{j'}\in \R^{d'\times r}$ with $\|U_j\|_{2}\vee\|V_{j'}\|_{2}\le 1$ such that
$$
\max_{M\in \R^{d\times d'},\|M\|_{2}\le 1,\text{rank}(M)\le r}\ 
$$
$$
\min_{j\le N_{dr,\epsilon/2}, j'\le N_{d'r,\epsilon/2}}\|M - U_jV_{j'}^\top\|_{2}\le \epsilon.
$$
Consequently, for any linear mapping $f$ and norm $\|\cdot\|_*$ in the range of $f$,
\begin{align}
\begin{aligned}
\label{lm-3-2}
& \sup_{M, \widetilde M\in \R^{d\times d'}, \|M-\widetilde M\|_{2}\le \epsilon
\atop{\|M\|_{2}\vee\|\widetilde M\|_{2}\le 1\atop
\text{rank}(M)\vee\text{rank}(\widetilde M)\le r}}
\frac{\|f(M-\widetilde M)\|_*}{\epsilon 2^{I_{r<d\wedge d'}}} \\
\le & \sup_{\|M\|_{2}\le 1\atop \text{rank}(M)\le r}\|f(M)\|_*
\le 2\max_{1\le j \le N_{dr,1/8}\atop 1\le j' \le N_{d'r,1/8}}\|f(U_jV_{j'}^\top)\|_*.
\end{aligned}
\end{align}
(iii) Given $\epsilon >0$, there exist $U_{j,k}\in \R^{d_k\times r_k}$
and $V_{j',k}\in \R^{d'_k\times r_k}$ with $\|U_{j,k}\|_{2}\vee\|V_{j',k}\|_{2}\le 1$ such that $\forall k\le K$

\small
\begin{align*}
\max_{M_k\in \R^{d_k\times d_k'},\atop \|M_k\|_{2}\le 1, \text{rank}(M_k)\le r_k}
\min_{j_k\le N_{d_kr_k,\epsilon/2} \atop j'_k\le N_{d_k'r_k,\epsilon/2}}
\Big\|\odot_{k=2}^K M_k - \odot_{k=2}^K(U_{j_k,k}V_{j_k',k}^\top)\Big\|_{\rm op}    
\end{align*}

$$
\le \epsilon (K-1).
$$
For any linear mapping $f$ and norm $\|\cdot\|_*$ in the range of $f$,
\begin{align}
\begin{aligned}
\label{lm-3-3}
& \sup_{M_k, \widetilde M_k\in \R^{d_k\times d_k'},\|M_k-\widetilde M_k\|_{2}\le\epsilon\atop
{\text{rank}(M_k)\vee\text{rank}(\widetilde M_k)\le r_k \atop \|M_k\|_{2}\vee\|\widetilde M_k\|_{2}\le 1\ \forall k\le K}}
\frac{\|f(\odot_{k=2}^KM_k-\odot_{k=2}^K\widetilde M_k)\|_*}{\epsilon(2K-2)} \\ 
 & \le \sup_{M_k\in \R^{d_k\times d_k'}\atop {\text{rank}(M_k)\le r_k \atop \|M_k\|_{2}\le 1, \forall k}}
\Big\|f\big(\odot_{k=2}^K M_k\big)\Big\|_*
\end{aligned}
\end{align}
and
\begin{align}
\begin{aligned}
\label{lm-3-4}
& \sup_{M_k\in \R^{d_k\times d_k'},\|M_k\|_{2}\le 1\atop \text{rank}(M_k)\le r_k\ \forall k\le K}
\Big\|f\big(\odot_{k=2}^K M_k\big)\Big\|_* \\
\le & 2\max_{1\le j_k \le N_{d_kr_k,1/(8K-8)}\atop 1\le j_k' \le N_{d_k'r_k,1/(8K-8)}}
\Big\|f\big(\odot_{k=2}^K U_{j_k,k}V_{j_k',k}^\top\big)\Big\|_*.
\end{aligned}
\end{align}
\end{lemma}

\begin{lemma}[Slepian’s inequality, \citet{vershynin2018high}]
Let $(X_t)_{t \in T}$ and $(Y_t)_{t \in T} \in T$ be two mean zero Gaussian Processes. Assume that for all $t,s \in T$, we have
\begin{equation}
\mathbb{E} X_t^2 = \mathbb{E} Y_t^2 \quad \text{and} \quad \mathbb{E}[(X_t - X_s)^2] \leq \mathbb{E}[(Y_t - Y_s)^2]. 
\end{equation}
Then for every $\tau \in R$, we have
\begin{equation}
    \mathbb{P} \left\{ \sup_{t \in T} X_t \geq \tau \right\} \leq \mathbb{P} \left\{ \sup_{t \in T} Y_t \geq \tau \right\}.
\end{equation}
Consequently,
\begin{equation}
    \mathbb{E} \sup_{t \in T} X_t \leq \mathbb{E} \sup_{t \in T} Y_t .
\end{equation}
\end{lemma}

\begin{lemma}[Sudakov-Fernique inequality, \citet{vershynin2018high}] \label{thm:Sudakov}
Let \((X_t)_{t \in T}\) and \((Y_t)_{t \in T}\) be two mean-zero Gaussian processes. Assume that, for all \(t, s \in T\), we have
\begin{equation}
\mathbb{E}[(X_t - X_s)^2] \leq \mathbb{E}[(Y_t - Y_s)^2].
\end{equation}
Then, 
\begin{equation}
\mathbb{E} \left[ \sup_{t \in T} \{X_t\} \right] \leq \mathbb{E} \left[ \sup_{t \in T} \{Y_t\} \right].
\end{equation}

\end{lemma}

A function $f:\R^{d}\to\R$ is $L$-Lipschitz continuous (sometimes called globally Lipschitz continuous) if there exists $L\in \R$ such that    
\begin{align*}
\|f(x)-f(y)\|\le L \| x -y \|.    
\end{align*}
The smallest constant $L > 0$ is denoted $\| f \|_{\rm Lip}$.
\begin{lemma}[Gaussian Concentration of Lipschitz functions, \citet{vershynin2018high}] \label{Lemma:Gaussian}
Let $X\sim N(0,I_d)$ and let $f:\R^{d}\to\R$ be a Lipschitz function. Then,
\begin{align*}
\P\left(f(X)-\E f(X) \ge t \right) \le 2 \exp\left( -\frac{t^2}{2 \|f\|_{\rm Lip}} \right)   .
\end{align*}
\end{lemma}

\begin{lemma}[Sigular values of Gaussian random matrices] \label{lemma:singular}
Let $X$ be an $n\times p$ matrix with i.i.d. $N(0, 1)$ entries. For any $x > 0$,
$\P\left( \left\| X^\top X  - n I_p \right\|_2 \ge p+ 2\sqrt{pn} + 2x(\sqrt{p}+\sqrt{n})+x^2 \right) \le  2 e^{- x^2/2}.$
\end{lemma}
\begin{proof}
It can be easily derived by Corollary 5.35 in \citet{vershynin2010introduction}.    
\end{proof}

\subsection{The Determination of Rank}
Here, the estimators are constructed with given ranks $r_1,\ldots,r_K$. Determining the number of factors in a data-driven manner has been a significant research topic in the factor model literature. Recently, \cite{han2022rank} established a class of rank determination approaches for tensor factor models, based on both the information criterion and the eigen-ratio criterion. These procedures can be adopted for our purposes.

\subsection{Tucker Decomposition and iterative TIPUP (iTIPUP) Algorithm}
The Higher-Order Orthogonal Iteration (HOOI) algorithm \cite{de2000multilinear} is one of the most popular methods for computing the Tucker decomposition, which relies on the orthogonal projection and Singular Value Decomposition (SVD) to update the estimation of loading matrices for each tensor mode in every iteration. 

    
Similar to the idea of HOOI, an iterative refinement version of TIPUP could further improve its performance, just as HOOI enhances the original Tucker decomposition method. The pseudo-code for this iterative refinement is provided in Algorithm \ref{alg:itipup}.

\subsection{Proof of Proposition \ref{thm:itipup}}
\begin{proof}
We focus on the case of $K=2$ as the iTIPUP begins with mode-$k$ matrix unfolding. In particular, we sometimes give explicit expressions only in the case of $k=1$ and $K=2$. For $K=2$, we observe a matrix time series with $X_t=A_1 F_t A_2^\top + \cE_t \in \mathbb R^{d_1\times d_2}$. Let $\overline\E(\cdot)=\E(\cdot|\{ \cF_1,...,\cF_n\})$ and $\overline\P(\cdot)=\P(\cdot|\{ \cF_1,...,\cF_n\})$.

Let \( L_k^{(m)} \) be the loss for \( \hat{A}_k^{(m)} \) at the \( m \)-th iteration,
\[
L_k^{(m)} = \| \widehat P_k^{(m)} - P_k \|_2, \quad
L^{(m)} = \max_{1 \leq k \leq K} L_k^{(m)},
\]
where \( \widehat P_k^{(m)} = A_k^{(m)} (A_k^{(m)\top} A_k^{(m)})^{-1} A_k^{(m)^\top} \).

We outline the proof as follows.
\paragraph{(i) Initialization.} Define
\[
{\rm TIPUP}_1 = \frac{1}{n} \sum_{t=1}^n X_t X_t^\top \in \mathbb{R}^{d_1 \times d_1},
\]
The ``noiseles'' version of TIPUP$_1$ is
\[
\Theta_1 = \frac{1}{n} \sum_{t=1}^n A_1 F_t A_2^\top A_2 F_1^\top A_1^\top \in \mathbb{R}^{d_1 \times d_1},
\]
which is of rank \( r_1 \). Then, by Davis-Kahan sin \( \Theta \) theorem (Lemma \ref{thm:sin}),
\begin{align} \label{eq:prop1}
L_1^{(0)} = \| P_1 - \widehat P_1^{(0)} \|_2 \le \frac{2 \| {\rm TIPUP}_1 - \Theta_1 \|_2}{\sigma_{r_1}(\Theta_1)}     .
\end{align}
To bound the numerator on the right-hand side of \eqref{eq:prop1}, we write
\begin{align*}
& {\rm TIPUP}_1 - \Theta_1 \\
= & \frac{1}{n} \sum_{t=1}^n A_1 F_t A_2^\top \cE_t^\top + \frac{1}{n} \sum_{t=1}^n \cE_t A_2 \cF_t^\top A_1^\top + \frac{1}{n} \sum_{t=1}^n \cE_t \cE_t^\top \\:= &  \Delta_1 + \Delta_2 + \Delta_3 .
\end{align*}
For \( \Delta_3 \), by Lemma \ref{lemma:singular},
\begin{align*}
&\mathbb{P} \left( \left\|  \sum_{t=1}^n \mathcal{E}_t \mathcal{E}_t^\top - n\mathbb{E} \mathcal{E}_t \mathcal{E}_t^\top \right\|_2 \right. \\ 
& \qquad \geq \left. d_1 + 2 \sqrt{d_1 d_2 n} + 2x(\sqrt{d_1}+\sqrt{d_2n}) + x^2\right)  \\
\leq & 2 e^{-x^2 / 2}  .
\end{align*}
By setting \( x \asymp \sqrt{d_1} \), with probability at least \( 1 - e^{-d_1} \),
\begin{equation*}
\left\| \Delta_3 - \mathbb{E} \mathcal{E}_t \mathcal{E}_t^\top \right\|_2 \lesssim \frac{d_1}{n} + \sqrt{\frac{d_1 d_2}{n}} .
\end{equation*}

Next consider \(\Delta_1\) and \(\Delta_2\). For any \( u \) and \( v \) in \( \R^{d_1} \), we have
\begin{equation*}
\overline{\mathbb{E}} \left( u^\top \left( \frac{1}{\sqrt{n}} \sum_{t=1}^n A_1 F_t A_2^\top \mathcal{E}_t^\top \right) v \right)^2 \leq \| \Theta_1 \|_2 \| u \|_2^2 \| v \|_2^2 .
\end{equation*}
Thus for \( u_i \) and \( v_i \) with \( \| u_i \|_2 = \| v_i \|_2 = 1 \), \( i = 1, 2 \),
\begin{align*}
& n\| \Theta_1 \|_2^{-1}  \overline{\mathbb{E}} \left( (u_1^\top \Delta_1 v_1 - u_2^\top \Delta_1 v_2 ) ^2 \right) \\
&\leq \left( \| u_1 - u_2 \|_2 \| v_1 \|_2 + \| u_2 \|_2 \| v_1 - v_2 \|_2 \right)^2\\
&\leq 2\left( \| u_1 - u_2 \|_2^2 + \| v_1-v_2 \|_2^2 \right)\\
&= 2 \mathbb{E} \left[ (u_1 - u_2)^\top \xi + (v_1 - v_2)^\top \zeta \right]^2 ,
\end{align*}
where \( \xi \) and \( \zeta \) are i.i.d \( N(0, I_{d_1}) \) vectors. As \( \Delta_1 \) is a \( d_1 \times d_1 \) Gaussian matrix under \( \overline{\mathbb{E}} \), the Sudakov-Fernique inequality (Lemma \ref{thm:Sudakov}) yields
\begin{align*}
& \sqrt{n} \| \Theta_1 \|_2^{-1/2} \overline{\mathbb{E}} \| \Delta_1 \|_2 \leq \sqrt{2} \mathbb{E} \sup_{\| u \|_2 = \| v \|_2 = 1} (u^\top \xi + v^\top \zeta) \\
= & \sqrt{2} \mathbb{E} \left( \| \xi \|_2 + \| \zeta \|_2 \right) \leq 2 \sqrt{2d_1} .
\end{align*}
It follows that
\begin{equation*}
\overline{\mathbb{E}} \| \Delta_1 \|_2 \leq \sqrt{\frac{8d_1}{n}} \| \Theta_1 \|_2^{1/2} .
\end{equation*}
Elementary calculation shows that \( \| \Delta_1 \|_2 \) is a \( \sqrt{n} \| \Theta_1 \|_2^{1/2} \)-Lipschitz function. Then, by Gaussian concentration inequality for Lipschitz functions (Lemma \ref{Lemma:Gaussian}),
\begin{equation*}
\overline{\mathbb{P}} \left( \| \Delta_1 \|_2 \geq \sqrt{\frac{8d_1}{n}} \| \Theta_1 \|_2^{1/2} + \sqrt{\frac{1}{n}} \| \Theta_1 \|_2^{1/2} x \right) \leq 2 e^{- \frac{x^2}{2}}   .
\end{equation*}
With \( x \asymp \sqrt{d_1} \), with probability at least \( 1 - e^{-d_1} \),
\begin{equation*}
\| \Delta_2 \|_2 = \| \Delta_1 \|_2 \lesssim \sqrt{\frac{d_1}{n}} \sigma_1(\Theta_1) .
\end{equation*}
As \( \sigma_1(\Theta_1) = \sigma_{r_1}(\Theta_1) \asymp \lambda^2 \), by \eqref{eq:prop1}, with probability at least \( 1 - e^{-d_1} \),
\begin{equation}\label{eq:prop1b}
L_1^{(0)} = \| \widehat P_1^{(0)} - P_1 \|_2 \lesssim \frac{1}{\lambda^2} \sqrt{\frac{d}{n}} +  \frac{d_1}{\lambda^2 n}  + \frac{1}{\lambda} \sqrt{\frac{d_1}{n}}  .
\end{equation}

\paragraph{(ii) Iterative refinement.} After the initialization with \( \widehat{A}_k^{(0)} \), the algorithm iteratively produces estimates \( \widehat{A}_k^{(m)} \) from \( m = 1 \) to \( m = J \). Define the matrix-valued operator TIPUP$_1(\cdot)$ as
\begin{equation*}
{\rm TIPUP_1}(U_2) = \frac{1}{n} \sum_{t=1}^n X_t U_2 U_2^\top X_t^\top \in \mathbb{R}^{d_2 \times d_2},
\end{equation*}
for any matrix-valued variable \( U_2 \in \mathbb{R}^{d_2 \times r_2} \). Given \( \widehat{A}_2^{(m)} \), the \((m+1)\)-th iteration produces estimates
\begin{align*}
\widehat{A}_1^{(m+1)} & = \text{LSVD}_{r_1}(\text{TIPUP}_1(\widehat{A}_2^{(m)})), \\ 
\widehat{P}_1^{(m+1)} & = \widehat{A}_1^{(m+1)} \widehat{A}_1^{(m+1)^\top} .
\end{align*}

The ``noiseless'' version of this update is given by
\begin{equation*}
\Theta_1(U_2) = \frac{1}{n} \sum_{t=1}^n A_1 F_t A_2^\top U_2 U_2^\top A_2 F_t^\top A_1^\top = \Theta_1
\end{equation*}
giving error-free ``estimates'',
\begin{equation*}
\widehat A_1^{(m+1)} = \text{LSVD}_{r_1}(\Theta_1(\widehat{A}_2^{(m)})), 
\end{equation*}
where \( \Theta_1(\widehat{A}_2^{(m)}) \) is of rank \( r_1 \). Thus, by  Davis-Kahan sin \( \Theta \) theorem (Lemma \ref{thm:sin})
\begin{equation} \label{eq:prop2}
L_1^{(m+1)} = \| \widehat{P}_1^{(m+1)} - P_1 \|_2 \leq \frac{2 \| \text{TIPUP}_1(\widehat{A}_2^{(m)}) - \Theta_1(\widehat{A}_2^{(m)}) \|_2}{\sigma_{r_1}[\Theta_1(\widehat{A}_2^{(m)})]} .
\end{equation}
To bound the numerator on the right-hand side of \eqref{eq:prop2}, we write
\begin{equation*}
\text{TIPUP}_1(U_2) - \Theta_1(U_2) := \Delta_1(U_2 U_2^\top) + \Delta_2(U_2 U_2^\top) + \Delta_3(U_2 U_2^\top)   ,
\end{equation*}
where for any \( M_2 \in \mathbb{R}^{d_2 \times d_2} \),
\begin{align*}
\Delta_1(M_2) &= \frac{1}{n} \sum_{t=1}^n A_1 F_t A_2^\top M_2 \mathcal{E}_t^\top, \\
\Delta_2(M_2) &= \frac{1}{n} \sum_{t=1}^n \mathcal{E}_t M_2 A_2 F_t^\top A_1^\top,\\
\Delta_3(M_2) &= \frac{1}{n} \sum_{t=1}^n \mathcal{E}_t M_2 \mathcal{E}_t^\top.
\end{align*}
As \(\Delta_i(M_2)\) is linear in \(M_2\), the numerator on the right-hand side of (\ref{eq:prop2}) can be bounded by
\begin{align*}
& \| \text{TIPUP}_1(\widehat{A}_2^{(m)}) - \Theta_1(\widehat{A}_2^{(m)}) \|_2 \\ 
\leq & \| \text{TIPUP}_1(A_2) - \Theta_1(A_2) \|_2 + L^{(m)} (2K-2) \sum_{i=1}^3 \| \Delta_i(M_2) \|_{\rm S},
\end{align*}
where 
\begin{equation*}
\| \Delta_i(M_2) \|_{\rm S} = \max_{\| M_2 \|_1 \leq 1, \text{rank}(M_2) \leq r_2} \| \Delta_i(M_2) \|_2.
\end{equation*}
We claim in certain events \(\Omega_i\), with \(\mathbb{P}(\Omega_i) \geq 1 - e^{-d_1}\),
\begin{align}
\| \Delta_i(M_2) \|_{\rm S} &\lesssim \sqrt{\frac{d_1+ d_2}{n}} \lambda, \quad i=1,2, \label{eq:prop3} \\
\| \Delta_3(M_2) \|_{\rm S} &\lesssim \frac{d_1 + d_2}{n} + \sqrt{\frac{d_1 + d_2}{n}}.  \label{eq:prop4}
\end{align}
We will also prove by induction, the denominator in (\ref{eq:prop2}) satisfies
\begin{equation} \label{eq:prop5}
2 \sigma_{r_1}[\Theta_1(\widehat{A}_2^{(m)})] \geq \sigma_{r_1}[\Theta_1(A_2)].
\end{equation}

Define the ideal version of the ratio in (\ref{eq:prop2}) as
\begin{align*}
L_k^{(\text{ideal})} & = \frac{\| \text{TIPUP}_1(\odot_{j \neq k} A_j) - \Theta_k(\odot_{j \neq k} A_j) \|_2}{\sigma_{r_1}[\Theta_k(\odot_{j \neq k} A_j)]}, \\ 
L^{(\text{ideal})} & = \max_{1 \leq k \leq K} L_k^{(\text{ideal})} .
\end{align*}
The proof of \eqref{eq:prop3} and \eqref{eq:prop4} implies that, in an event with probability at least \( 1 - \frac{k}{2} e^{-d_k} \),
\begin{equation} \label{eq:prop6}
L_k^{(\text{ideal})} \lesssim \frac{1}{\lambda^2} \sqrt{\frac{d_k}{n}} + \frac{d_k}{\lambda^2 n} + \frac{1}{\lambda} \sqrt{\frac{d_k}{n}}.
\end{equation}
Putting together \eqref{eq:prop1}, \eqref{eq:prop1b}, \eqref{eq:prop2}, \eqref{eq:prop3}, \eqref{eq:prop4} and \eqref{eq:prop5}, we have
\begin{equation*}
L_k^{(m+1)} \leq L_k^{(\text{ideal})} + L^{(m)} \rho,
\end{equation*}
with \(\rho < 1\). By induction,
\begin{equation*}
L^{(m+1)} \leq (1 + \rho +\dots+ \rho^m) L^{(m)} + \rho^{m+1} L^{(0)}.
\end{equation*}
We achieve the desired results.

We divided the rest of the proof into 4 steps to prove \eqref{eq:prop3} and \eqref{eq:prop4} for $i = 1,2,3$ and (\ref{eq:prop5}).

\paragraph{Step 1.} We prove (\ref{eq:prop3}) for the \(\| \Delta_1(M_2) \|_{\rm S}\).

By Lemma \ref{lemma:epsilonnet} (ii), there exists \(\widetilde{M}^{(\ell, \ell')} \in \mathbb{R}^{d_2 \times d_2}\) of the forms \(W_\ell Q_{\ell'}^\top\) with \(W_\ell \in \mathbb{R}^{d_2 \times r_2}\), \(Q_{\ell'} \in \mathbb{R}^{d_2 \times r_2}\), \(1 \leq \ell, \ell' \leq N_{d_2 r_2,1/8} = 17^{d_2 r_2}\), such that $\| \widetilde{M}^{(\ell, \ell')} \|_2 \leq 1$, $\text{rank}(\widetilde{M}^{(\ell, \ell')}) \leq r_2$ and
\begin{equation*}
\| \Delta_1(M_2) \|_{\rm S} \leq 2 \max_{\ell, \ell' \leq N_{d_2 r_2,1/8}} \left\| \frac{1}{n} \sum_{t=1}^n A_1 F_t A_2^\top \widetilde{M}^{(\ell, \ell')} \mathcal{E}_t^\top \right\|_2.
\end{equation*}
Elementary calculation shows that $\| \Delta_1(\widetilde{M}^{(\ell, \ell')}) \|_2$ is a \(\sqrt{n} \| \Theta_1 \|_2\)-Lipschitz function of \(\mathcal{E}_1, \ldots, \mathcal{E}_n\). Employing similar argument as the initialization (Sudakov-Fernique inequality, Lemma \ref{thm:Sudakov}), we have
\begin{equation*}
\overline{\mathbb{E}} \left\| \frac{1}{n} \sum_{t=1}^n A_1 F_t A_2^\top \widetilde{M}^{(\ell, \ell')} \mathcal{E}_t^\top \right\|_2 \leq \sqrt{\frac{8d_1}{n}} \| \Theta_1 \|_2^{1/2}.
\end{equation*}
Then, by Gaussian concentration inequality for Lipschitz functions (Lemma \ref{Lemma:Gaussian}),
\begin{align*}
& \overline{\mathbb{P}} \left( \left\| \frac{1}{n} \sum_{t=1}^n A_1 F_t A_2^\top \widetilde{M}^{(\ell, \ell')} \mathcal{E}_t^\top \right\|_2 - \sqrt{\frac{8d_1}{n}} \| \Theta_1 \|_2^{1/2} \right. \\
& \qquad\geq \left. \frac{1}{\sqrt{n}} \| \Theta_1 \|_2^{1/2} x \right) \\
\leq &  e^{-x^2 / 2}.
\end{align*}
Hence,
\begin{align*}
& \overline{\mathbb{P}} \left( \| \Delta_1(M_2) \|_{\rm S} \geq \sqrt{\frac{8d_1}{n}} \| \Theta_1 \|_2^{1/2} + \frac{1}{\sqrt{n}} \| \Theta_1 \|_2^{1/2} x \right) \\
\leq & N_{d_2 r_2,1/8}^2 e^{-x^2 / 2}.
\end{align*}
This implies that with \( x \asymp \sqrt{d_2 r_2} \) that in an event with probability at least \( 1 - e^{-d_2} \),
\begin{equation*}
\| \Delta_1(M_2) \|_{\rm S} \lesssim \sqrt{\frac{d_1+d_2 r_2}{n}} \lambda \lesssim \sqrt{\frac{d_1+d_2}{n}} \lambda.
\end{equation*}

\paragraph{Step 2.} The bound of \(\| \Delta_2(M_2) \|_{\rm S}\) follows from the same argument as the above step.

\paragraph{Step 3.} Here we prove (\ref{eq:prop4}) for \(\|\Delta_3(M_2)\|_{\rm S}\).
By Lemma \ref{lemma:epsilonnet}(iii), we can find \(U_2^{(\ell)} \in \mathbb{R}^{d_2 \times r_2}\), \(1 \leq \ell \leq N_{d_2 r_2, 1/8}\) such that $\| U_2^{(\ell)} \|_2 \leq 1$ and
\begin{equation*}
\| \Delta_3(M_2) \|_{\rm S} \leq 2 \max_{1 \leq \ell \leq N_{d_2 r_2, 1/8}} \left\| \frac{1}{n} \sum_{t=1}^n \mathcal{E}_t U_2^{(\ell)} U_2^{(\ell)\top} \mathcal{E}_t^\top \right\|_2.
\end{equation*}
By Lemma \ref{lemma:singular},
\begin{align*}
& \mathbb{P} \left( \left\|  \sum_{t=1}^n \mathcal{E}_t U_2^{(\ell)} U_2^{(\ell)\top} \mathcal{E}_t^\top - n\mathbb{E} \mathcal{E}_t U_2^{(\ell)} U_2^{(\ell)\top} \mathcal{E}_t^\top \right\|_2 \right. \\
& \qquad \geq  \left. d_1 + 2 \sqrt{d_1 r_2 n} + 2 x (\sqrt{d_1} + \sqrt{r_2 n}) + x^2 \right) \\
\leq & 2 e^{-x^2 / 2}.
\end{align*}
Hence, by setting \( x \asymp \sqrt{d_2 r_2} \), we have with probability at least \( 1 - e^{-d_2} \),
\begin{align*}
\| \Delta_3(M_2) \|_{\rm S} & \lesssim \frac{d_1}{n} + \sqrt{\frac{d_1 r_2}{n}} + \frac{d_2 r_2}{n} + \sqrt{\frac{d_2 r_2^2}{n^2}} \\
& \lesssim \frac{d_1 + d_2}{n} + \sqrt{\frac{d_1 + d_2}{n}}.
\end{align*}

\paragraph{Step 4.} Next, we prove (\ref{eq:prop5}). Note that
\begin{align*}
& \| \Theta_1(\widehat{A}_2^{(m)}) - \Theta_1(A_2) \|_2  \\ 
&\leq \left\| \frac{1}{n} \sum_{t=1}^n A_1 F_t A_2^\top  (\widehat{A}_2^{(m)} \widehat{A}_2^{(m)\top} - A_2 A_2^\top) A_2 F_t^\top A_1^\top \right\|_2 \\
&\leq \| \widehat{A}_2^{(m)} \widehat{A}_2^{(m)\top} - A_2 A_2^\top \|_2 \| \Theta_1 \|_2 \\
&\leq L^{(m)} \| \Theta_1 \|_2  .
\end{align*}
Thus, by Weyl's inequality (Lemma \ref{thm:weyl}),
\begin{equation*}
\sigma_{r_1}[\Theta_1(\widehat{A}_2^{(m)})] \geq \sigma_{r_1}[\Theta_1(A_2)] - L^{(m)} \| \Theta_1 \|_2 \geq \frac{1}{2} \sigma_{r_1}[\Theta_1(A_2)],
\end{equation*}
when \(\min_{1 \leq k \leq K} \frac{\sigma_{r_k}[\Theta_k(A_2)]}{\| \Theta_k(A_2) \|_2} \geq 2 L^{(m)}\). 
We prove this condition by induction.
\end{proof}

\section{Appendix B: Supplemental Numerical Results}
\label{append:data}

\subsection{Supplement to FAO Data Analysis}

Studies on FAO data can reveal temporal causal links that help inform decisions, policies, and investments related to food and agriculture. The dataset has also been examined in other studies (e.g., \citet{chepeliev2020gtap, lee2023statistical, chevuru2023copernicus}) to address similar issues discussed in this article. Predicting yield and production in South America based on data from East Asia, North America and Europe can aid in analyzing global climate effects. For example, El Ni~no causes droughts in Asia and heavy rains in South America. Studying El Ni~no’s impact on crop yields in East Asia and North America helps anticipate its effects on South American agriculture \citep{anderson2017life}. Such predictions also enhance our understanding of global market dynamics. South America and the US are top soybean producers, and changes in production on one continent influence prices and planting decisions on the other. A US drought reducing soybean yields could drive up prices, prompting South American farmers to increase cultivation, affecting future yields \citep{robinson2015international}.

Table \ref{tab:FAO-statistics} illustrates detailed information about the lists of crop types, livestock types, countries, and metrics involved in the two prediction tasks using FAO data:
\begin{itemize}
    \item[(i)] (crop $\sim$ crop) using Yield and Production data of 11 different crops for 33 countries in East Asia, North America, and Europe (i.e., $\cX_t \in \mathbb{R}^{33\times2\times11}$) to predict the Yield and Production quantities of the same crops for 13 countries in South America (i.e., $\cY_t \in \mathbb{R}^{13\times2\times11}$);
    \item[(ii)] (crop $\sim$ livestock) using four agricultural statistics (Producing Animals, Animals-slaughtered, Milk Animals, Laying) associated with 5 kinds of livestock of 26 selected countries in Europe (i.e., $\cX_t \in \mathbb{R}^{26\times4\times5}$) to predict three metrics (Area-harvested, Production, and Yield) of 11 crops in the same countries (i.e., $\cY_t \in \mathbb{R}^{26\times3\times11}$).
\end{itemize}

Figures \ref{fig:FAO}, \ref{fig:supp-fao1} and \ref{fig:supp-fao2}, visualize the comparisons between ground-truth values and predicted metrics in the two prediction tasks for FAO data. The predicted and actual quantities are represented by the color intensity of each country. A closer examination reveals that our proposed FATTNN model produces predictions that more closely align with the ground truth compared to the benchmark methods (right panel). The color gradients in the FATTNN predictions more accurately reflect the actual data, indicating superior spatial prediction performance. We have also included numerical values tabularly reported in Table \ref{tab:supp1} and Table \ref{tab:supp2}, providing a quantitative comparison that reflects the advantages of our method.

\subsection{Supplement to NYC Taxi Data Analysis}
Factorized temporal tensor acts like time-varying principal components, summarizing the key information of data in reduced dimensionality. Specific interpretation of core tensor varies across datasets. For example, in taxi data, the dimension reduction from (19,19) (pick-up, drop-off districts) in raw data to (4,4) in core tensor reflects clusters of geographical regions that share similar traffic patterns, and the reduction from 8 (hours) in raw data to 2 in core tensor summarizes traffic patterns into rush hour and off-peak hour.

Figure \ref{fig:ManhattanMap} provides a detailed Manhattan district map together with time series plots of hourly traffic trends in Midtown and Downtown Manhattan. The district information is used in segmenting districts into different prediction tasks in NYC Taxi data analysis. We could observe a slightly decreasing trend from the beginning of 2014 to the end of 2015 in both midtown and downtown.

Figure \ref{fig:taxi} (with an enlarged version shown in Figure \ref{fig:taxi-large}) presents a comparison between ground-truth values and predicted pick-up and drop-off volumes using various methods. In this figure, the ground truth is depicted by the blue solid line, while our FATTNN predictions are shown as a dark red line with smaller dots. The FATTNN line closely follows the ground truth, both in shape and location, outperforming the benchmark method represented by the red line. This demonstrates our model's enhanced ability to capture temporal patterns in the New York taxi data.

\subsection{Supplement to FMRI Data Analysis}

Figure \ref{fig:fmri} presents a comparison between ground-truth values and predicted outputs for selected samples in FMRI dataset. 
The figure shows that the FMRI image predicted by FATTNN closely resembles the actual FMRI image shown on the left. In contrast, the images generated by TCN, LSTM, and TRL without considering tensor factor structures, shows a more blur pattern and noticeable discrepancies. The Conv-TT-LSTM also produces a plot that resembles the actual FMRI image well. The figure is also consistent with the test MSE results summarized in Table \ref{tab:realDataAccuracy}, that Conv-TT-LSTM yields the lowest MSE in the FMRI prediction task, followed by the FATTNN. This is not surprising as Conv-TT-LSTM is specifically designed for image prediction. We also observe a noticeable difference in computation time. Though Conv-TT-LSTM has the best prediction accuracy, it is quite computationally heavy - about 50 times slower than FATTNN.

\section{Appendix C: An Extended Discussion \\ on the Generalizations of FATTNN}
\label{append:discussion}

In this article, we focus on tensor-on-tensor time series forecasting. The integration of TIPUP and TCN, both designed specifically for handling temporal data, makes the proposed approach a powerful tool for forecasting tensor series over time. 
For non-temporal data, the TIPUP-TCN-based FATTNN is still applicable for tensor-on-tensor prediction tasks, but the strengths of TIPUP and TCN may not be fully utilized. This limitation can be easily rectified by adopting proper types of tensor factor models and neural networks in the implementation of FATTNN, depending on the nature of the prediction tasks and data types involved.

\emph{Our proposed FATTNN framework is not confined to temporal data but can accommodate a variety of data types}, though we presented our models for the purpose of tensor-on-tensor time series forecasting. 
The main idea of FATTNN is to use a low-rank tensor factor model to capture the intrinsic patterns among the observed covariates, and then proceed with a neural network to model the intricate relationships between covariates and responses. With appropriate choices of factor models and neural network architectures, the proposed FATTNN can be naturally adapted to any simple (e.g., i.i.d observations) or complex tensor-type data (e.g., image, graph, network, text, video). 
When handling time series tensors, we adopt TIPUP factorization and TCN to fully capitalize on their temporal nature. For higher-order tensors (e.g. $K > 4$ ), Tensor-Train \citep{oseledets2011tensor} and Hierarchical Tucker decompositions \citep{lubich2013dynamical} may be better suited. \emph{To extend beyond time series data, different types of factor models and neural network architectures should be deployed, depending on the nature of the data.} For example, for graph data, Graph Neural Network is a more suitable alternative to TCN. For text data, Transformer models could substitute for the TCN in our FATTNN framework. Our proposed FATTNN is flexible in accommodating different types of tensor factor models and deep learning architectures for tensor-on-tensor prediction using diverse data types.

\begin{algorithm*} 
\caption{Iterative TIPUP (iTIPUP) algorithm} \label{alg:itipup} 
\begin{algorithmic}[1]
\STATE \textbf{Input:} \( \mathcal{X}_t \in \mathbb{R}^{d_1 \times \ldots \times d_K} \) for \( t = 1, \ldots, n \), rank \( r_k \) for all \( k = 1, \ldots, K \), the tolerance parameter \( \epsilon > 0 \), the maximum number of iterations \( J \).
    \STATE Let \( j = 0 \), initiate via TIPUP on $\cX_1,...,\cX_n$ to obtain
\begin{align*}
\widehat A_k^{(0)}=\text{LSVD}_{r_k}  \left( \frac{1}{n} \sum_{t=1}^{n} \mat_k(\cX_{t}) \mat_k^\top(\cX_t)\right),    
\end{align*}
where LSVD$_{r_k}$ stands for the top $r_k$ left singular vectors, $k=1,...,K$.

    \REPEAT 
    \STATE Let $j = j + 1$. 
    \FOR{$k = 1, \ldots, K$}
    \STATE Given previous estimates $\widehat A_1^{(j-1)}, \ldots, \widehat A_{k-1}^{(j-1)}, \widehat A_{k+1}^{(j-1)}, \ldots, \widehat A_K^{(j-1)}$, sequentially calculate,
\begin{align*}
Z_{t,k}^{(j)} = \cX_t \times_1 \widehat A_1^{(j)\top} \times_2 \ldots \times_{k-1} \widehat A_{k-1}^{(j)\top}  \times_{k+1} \widehat A_{k+1}^{(j-1)\top} \times_{k+2} \ldots \times_K \widehat A_K^{(j-1)\top},    
\end{align*}
    for $t = 1, \ldots, n$. 
    \STATE Perform TIPUP on the new tensor series $(Z_{1,k}^{(j)}, \ldots, Z_{n,k}^{(j)})$,
\begin{align*}
\widehat A_k^{(j)}=\text{LSVD}_{r_k}  \left( \frac{1}{n} \sum_{t=1}^{n} \mat_k(Z_{t,k}^{(j)}) \mat_k^\top(Z_{t,k}^{(j)})\right).    
\end{align*}
    \ENDFOR
    \UNTIL $j = J$ or
    $\max_{1 \leq k \leq K} \|\widehat A_k^{(j)} \widehat A_k^{(j)\top} - \widehat A_k^{(j-1)} \widehat A_k^{(j-1)\top}\|_2 \leq \epsilon$.
    \STATE \textbf{Output:}
    \begin{align*}
        \widehat A_k &= \widehat A_k^{(j)}, \quad k = 1, \ldots, K, \\
        \widehat F_t &= \mathcal{X}_t \times_{k=1}^K \widehat A_K^\top, \quad t = 1, \ldots, n.
    \end{align*}
\end{algorithmic}
\end{algorithm*}

\begin{table*}
    \centering
    \caption{Details of statistics and metrics involved in the FAO data analysis}
    \begin{tabular}{c|c|c}
     \addlinespace[5pt]
    \multicolumn{3}{c}{(a) Prediction Task (i): crop $\sim$ crop}  \\
    \hline
    \addlinespace[2pt]
    \multicolumn{1}{c}{}& $\mathcal{X}_t \in \mathbb{R}^{33\times2\times11}$ & $\mathcal{Y}_t\in \mathbb{R}^{13\times2\times11}$ \\
    \hline
    Countries
    & \multicolumn{1}{p{5cm}|}{AUT, BGR, CAN, CHN, HRV, CYP, CZE, PRK, DNK, EST, FIN, FRA, DEU, GRC, HUN, IRL, ITA, JPN, LVA, LTU, LUX, MLT, MNG, NLD, POL, PRT, KOR, ROU, SVK, SVN, ESP, SWE, USA (33 countries)}
    & \multicolumn{1}{p{5cm}}{ARG, BOL, BRA, CHL, COL, ECU, GUF, GUY, PRY, PER, SUR, URY, VEN (13 countries)} \\
    \hline
    Crop/Livestock selected 
    & \multicolumn{1}{p{5cm}|}{Cereals, Citrus Fruit, Fibre Crop, Fruit, OliCrops cake equivalent, OliCrops oil equivalent, Pulses, Roots and tubers, Sugar crop, Treenuts \quad (11 types)}
    & \multicolumn{1}{p{5cm}}{the same as in $\mathcal{X}_t$} \\
    \hline
    Metrics involved
    & Yield, Production (2 metrics)
    & \multicolumn{1}{p{5cm}}{the same as in $\mathcal{X}_t$} \\
    \hline
    \addlinespace[10pt]
    \multicolumn{3}{c}{(b) Prediction Task (i): crop $\sim$ livestock}  \\
    \hline
    \addlinespace[2pt]
    \multicolumn{1}{c}{}& $\mathcal{X}_t \in \mathbb{R}^{26\times4\times5}$ & $\mathcal{Y}_t\in \mathbb{R}^{26\times3\times11}$ \\
    \hline
    Countries
    & \multicolumn{1}{p{5cm}|}{AUT, BGR, HRV, CYP, CZE, DNK, EST, FIN, FRA, DEU, GRC, HUN, IRL, ITA, LVA, LTU, LUX, MLT, NLD, POL, PRT, ROU, SVK, SVN, ESP, SWE (26 countries)}
    & \multicolumn{1}{p{5cm}}{the same as in $\mathcal{X}_t$} \\
    \hline
    Crop/Livestock selected 
    & \multicolumn{1}{p{5cm}|}{Beef and Buffalo Meat, Eggs, Poultry Meat, Milk, Sheep and Goat Meat (5 types)} 
    & \multicolumn{1}{p{5cm}}{Cereals, Citrus Fruit, Fibre Crop, Fruit, OliCrops cake equivalent, OliCrops oil equivalent, Pulses, Roots and tubers, Sugar crop, Treenuts \quad (11 types)} \\
    \hline
    Metrics involved
    & \multicolumn{1}{p{5cm}|}{Producing Animals, Animals-slaughtered, Milk Animals, Laying  (4 metrics)}
    & \multicolumn{1}{p{5cm}}{Yield, Production, Area-harvested (3 metrics)} \\
    \hline
    \end{tabular}
\label{tab:FAO-statistics}

Note: A full description of the code definitions is available on the United Nations Food and Agriculture Organization Crops and Livestock Products Database website. 
\end{table*}
\clearpage

\begin{figure*}
    \centering
    \includegraphics[width=1.0\textwidth]{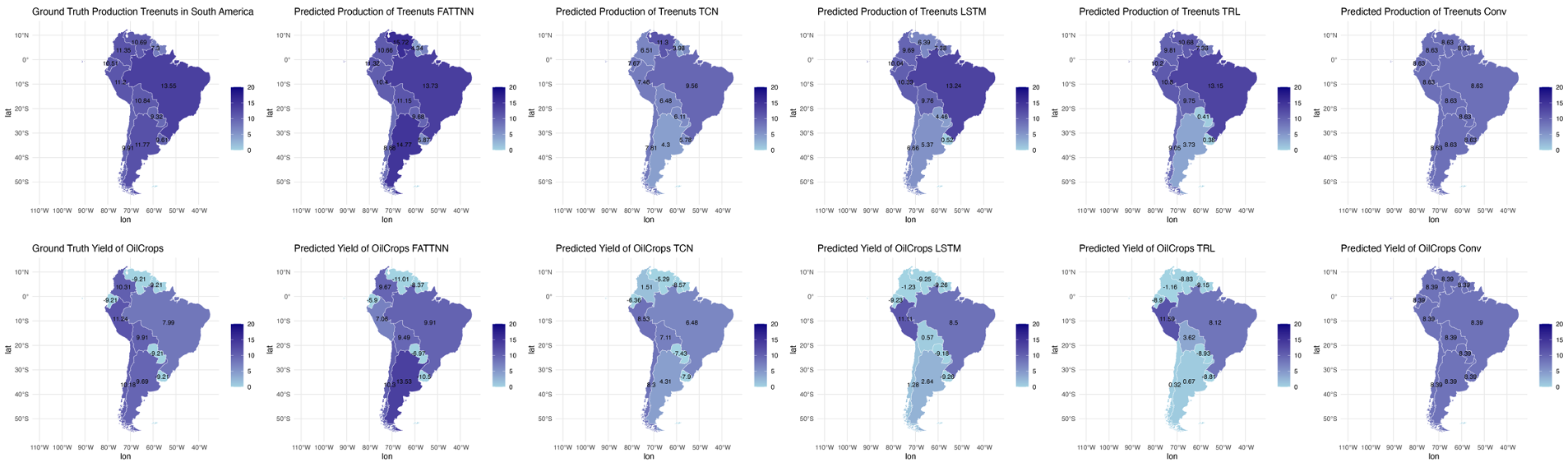} 
    \caption{Comparisons between ground-truth values and predicted Production of Treenuts (top row) and Yield of Oilcrops (bottom row) in South America using the crop data of 33 countries in East Asia, North America, and Europe. Values are plotted on the log-transformed scale. From left to right: Ground truth, FATTNN, TCN, LSTM, TRL, and Conv-TT-LSTM.}
    \label{fig:supp-fao1}
\end{figure*}

\begin{figure*}
    \centering
    \includegraphics[width=1.0\textwidth]{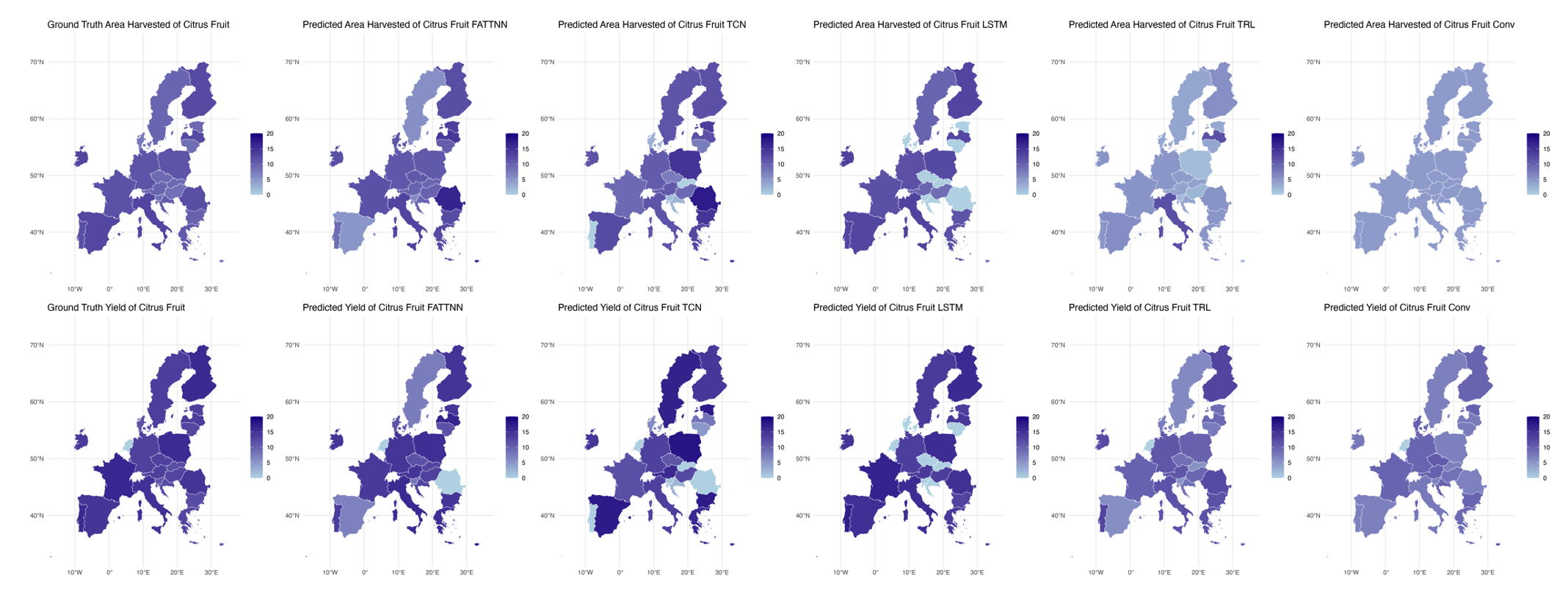}
    \caption{Comparisons between ground-truth values and predicted Area-harvested (top row) and Yield (bottom row) of citrus fruit in 26 selected countries in Europe using the livestock data of the same 26 countries. Values are plotted on the log-transformed scale. From left to right: Ground truth, FATTNN, TCN,  LSTM, TRL, and Conv-TT-LSTM.}
    \label{fig:supp-fao2}
\end{figure*}
\clearpage

\begin{table*}
\centering
\captionsetup{justification=centering} 
\caption{Numerical values of ground-truth and predicted (a) Production of  Treenuts and (b) yield Oilcrops in 13 South America countries using crop data of 33 countries in East Asia, North America, and Europe - a supplement to Figure \ref{fig:FAO}.} 
\label{tab:supp1}
\begin{subtable}{0.6\textwidth}
\captionsetup{justification=centering} 
\caption{\normalsize Predicted Production of Treenuts} \label{tab:predicted_prod}
\begin{tabular}{l|ccc}
\hline
Country & Ground Truth  & FATTNN & TCN ($\mathcal{Y}\sim \mathcal{X}$) \\
\hline
Argentina&11.77&14.77&4.30\\ 
Bolivia&10.84&11.15&6.48\\
Brazil&13.55&13.73&9.56\\ 
Chile&9.91&8.88&7.81\\
Colombia&11.35&10.66&6.51\\
Ecuador&10.51&11.32&7.67\\
French Guiana&-9.21&-0.71&2.80\\
Guyana&7.30&4.34&3.98 \\
Paraguay&9.32&9.68&6.11 \\ 
Peru&11.20&10.40&7.46 \\
Suriname&7.42&7.03&6.02 \\ 
Uruguay&9.61&5.87&5.78 \\ 
Venezuela&10.69&15.72&11.30\\ 
\hline
\end{tabular}
\medskip
\end{subtable}

\hfill

\begin{subtable}{0.6\textwidth}
\captionsetup{justification=centering} 
\caption{\normalsize Predicted Yield of Oilcrops}
\label{tab:predicted_prod2}
\begin{tabular}{l|ccc}
\hline
Country & Ground Truth  & FATTNN & TCN ($\mathcal{Y}\sim \mathcal{X}$) \\ 
\hline
Argentina     & 9.69 & 13.53  & 4.31 \\
Bolivia    & 9.91 & 9.49  & 7.11 \\
Brazil     & 7.99 & 9.91  &  6.48  \\ 
Chile      & 10.18 & 10.30  & 8.30 \\ 
Colombia     & 10.31 & 9.67  & 1.51 \\ 
Ecuador     & -9.21 &  -5.90  &  -6.36 \\ 
French Guiana     & -9.21 & -9.99  & -10.10 \\ 
Guyana     & -9.21 & -8.37  & -8.57 \\ 
Paraguay      & -9.21 & -6.97  & -7.43 \\
Peru     & 11.24 & 7.06  & 8.53 \\ 
Suriname      & -9.21 & -8.32  & -5.64 \\ 
Uruguay     & -9.21 & -10.50  & -7.90 \\ 
Venezuela     & -9.21 & -11.01  & -5.29 \\ 
\hline
\end{tabular}
\end{subtable}
\end{table*}

\begin{table*}
\centering
\captionsetup{justification=centering} 
\caption{Numerical values of ground-truth and (a) predicted Area-Harvested of citrus fruit (b) predicted Yield of citrus fruit, in 26 European countries using livestock data from the same 26 countries - a supplement to Figure \ref{fig:FAO}.} 
\label{tab:supp2}
\begin{subtable}{0.45\textwidth}
\caption{\normalsize Predicted Area-Harvested of Citrus Fruit }
\label{tab:predicted_prod2}
\begin{tabular}{l|ccc}
\hline
Country & Ground Truth  & FATTNN & TCN ($\mathcal{Y}\sim \mathcal{X}$) \\ 
\hline
Austria     & 10.01 & 10.74 & 13.54 \\
Bulgaria    &  9.36 & 11.81 & 13.65 \\
Croatia     &  9.18 & 10.03 &  0.33 \\
Cyprus      &  8.39 & 12.07 & 11.58 \\
Czechia     & 10.07 &  9.54 &  7.65 \\
Denmark     &  8.55 &  7.98 &  1.90 \\
Estonia     &  9.99 & 10.53 & 13.61 \\
Finland     & 12.14 & 14.00 & 11.41 \\
France      & 12.44 & 11.48 &  9.40 \\
Germany     & 10.07 & 11.60 & 10.08 \\
Greece      &  9.65 & 12.03 & 12.86 \\
Hungary     &  9.14 & 11.37 &  8.96 \\
Ireland     & 10.84 & 12.24 &  9.91 \\
Italy       & 11.82 & 11.87 &  9.86 \\
Latvia      & 10.71 & 10.94 &  5.96 \\
Lithuania   &  9.75 &  8.53 &  4.28 \\
Luxembourg  & 10.05 &  8.26 &  3.16 \\
Malta       &  6.42 & -1.43 & -0.02 \\
Netherlands &  9.65 & 10.44 &  7.39 \\
Poland      & 11.98 & 12.27 & 15.52 \\
Portugal    & 12.61 & 12.11 & 22.19 \\
Romania     & 12.02 & 14.08 & 18.88 \\
Slovakia    &  9.00 &  6.94 & -0.97 \\
Slovenia    &  8.20 &  6.61 &  0.83 \\
Spain       & 11.17 & 7.22  & 14.15 \\
Sweden      & 10.10 & 7.71  & 13.22 \\
\hline
\end{tabular}
\end{subtable}
\hfill
\begin{subtable}{0.45\textwidth}
\caption{\normalsize Predicted Yield of Citrus Fruit }
\label{tab:predicted_prod2}
\begin{tabular}{l|ccc}
\hline
Country & Ground Truth  & FATTNN & TCN ($\mathcal{Y}\sim \mathcal{X}$) \\ 
\hline
Austria     & 13.47 & 13.94 & 15.92  \\
Bulgaria    & 12.14 & 15.01 & 17.34      \\
Croatia     & 11.98 & 12.52 &  2.38     \\
Cyprus      & 11.52 & 14.60 & 14.80     \\
Czechia     & 13.38 & 10.68 & 11.13    \\
Denmark     & 11.52 &  9.43 &  5.97    \\
Estonia     & 13.29 & 13.06 & 17.61    \\
Finland     & 15.86 & 16.91 & 14.45   \\
France      & 16.19 & 14.86 & 11.88   \\
Germany     & 13.27 & 15.24 & 13.31    \\
Greece      & 12.87 & 14.67 & 16.19    \\
Hungary     & 12.78 & 14.21 & 12.57    \\
Ireland     & 14.15 & 14.87 & 14.59   \\
Italy       & 15.07 & 14.31 & 13.43   \\
Latvia      & 13.72 & 13.56 &  8.04    \\
Lithuania   & 12.68 &  9.81 &  4.78    \\
Luxembourg  & 12.77 &  9.49 &  3.20    \\
Malta       &  9.76 & -0.76 & -0.49  \\
Netherlands & 12.12 & 12.44 &  8.33    \\
Poland      & 15.74 & 15.40 & 18.19   \\
Portugal    & 15.86 & 16.34 & 25.16    \\
Romania     & 14.76 & 16.14 & 20.95    \\
Slovakia    & 11.99 &  8.17 & -0.96    \\
Slovenia    & 11.36 &  8.58 &  0.84    \\
Spain       & 14.62 &  9.52 & 17.50   \\
Sweden      & 13.63 & 10.19 & 16.74   \\
\hline
\end{tabular}
\end{subtable}
\end{table*}

\clearpage

\begin{figure*}
    \centering
    \includegraphics[width=1.0\textwidth]{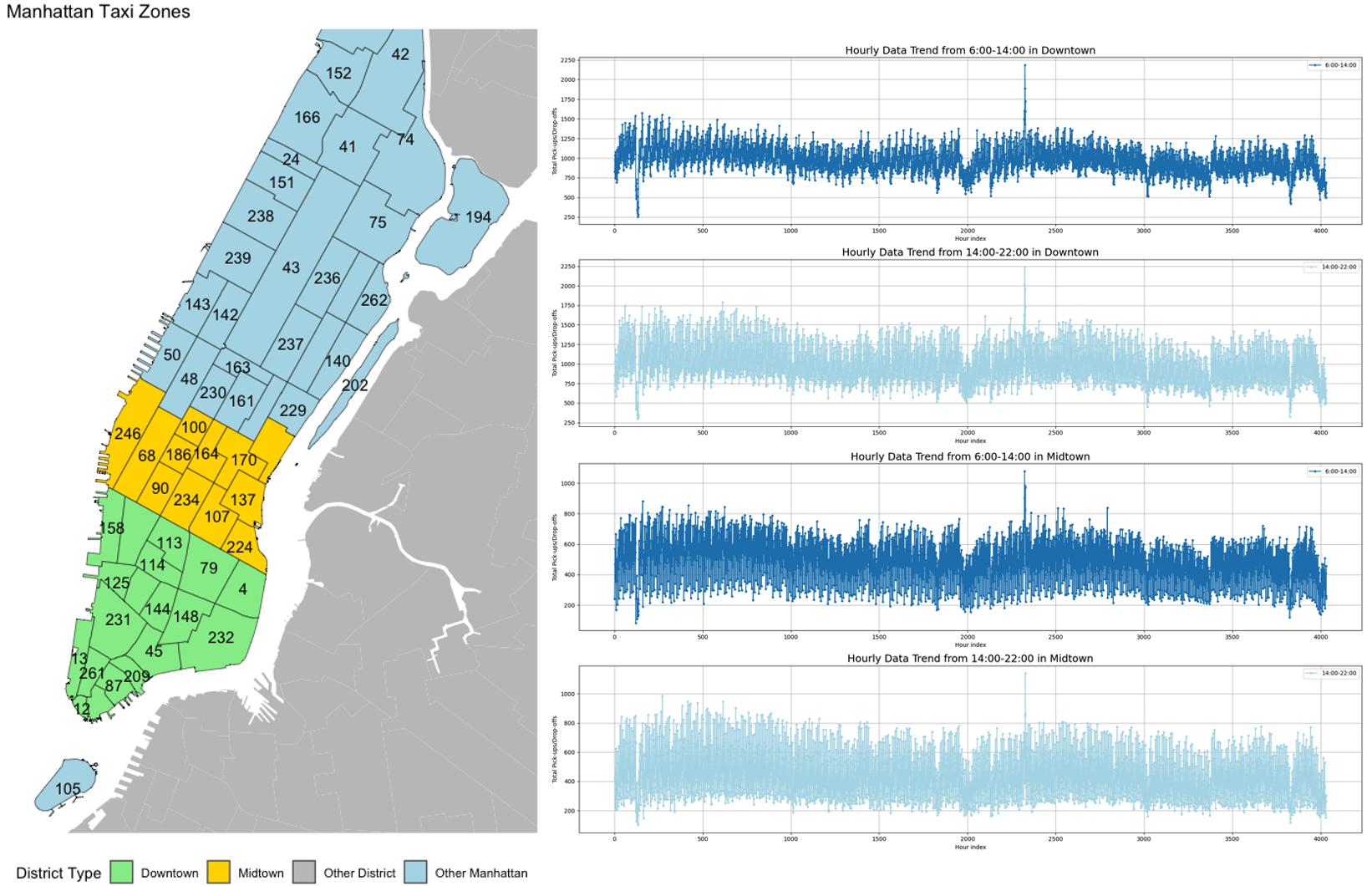} 
    \caption{Left panel: An overview of 69 districts in Manhattan. Districts colored with gold represent Midtown and districts colored with green are categorized as Downtown Manhattan. Right panel: Overall trend of the sum of pick-up and drop-off by hour in Midtown and Downtown Manhattan.}
    \label{fig:ManhattanMap}
\end{figure*} 

\clearpage

\begin{figure*}
    \centering
    \includegraphics[width=\textwidth]{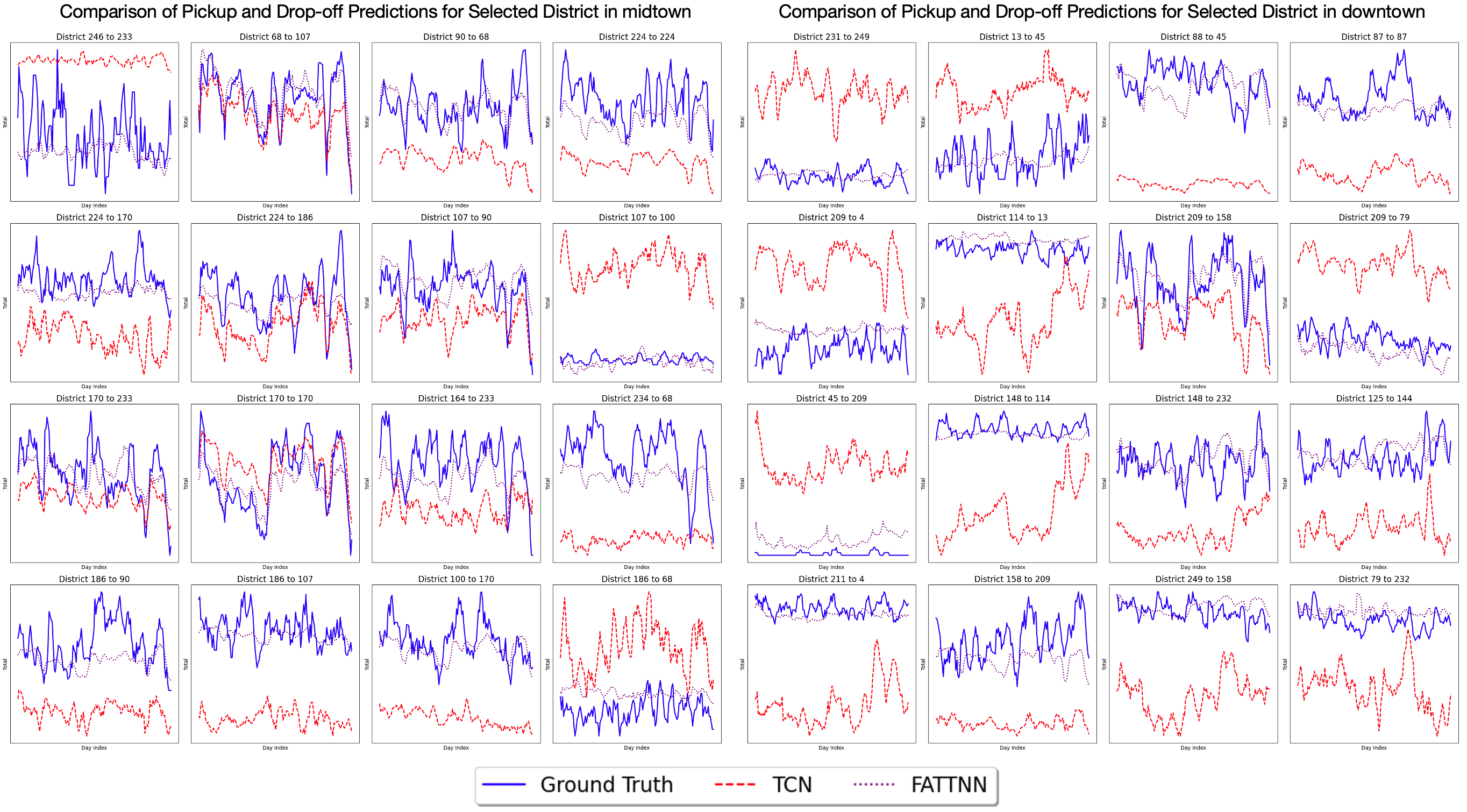} 
    \caption{Comparisons between ground-truth values and predicted pick-up and drop-off volumes using various methods. This plot is an enlarged version of Figure \ref{fig:taxi} of the main context. ``District A to B'' denotes the traffic volume that passengers were picked up in District A and dropped off in District B. The district numbering is assigned according to the Manhattan district map shown in Figure \ref{fig:ManhattanMap}. Left: Midtown Manhattan; Right: Downtown Manhattan. 
    }
\label{fig:taxi-large}
\end{figure*}

\begin{figure*}
    \centering
    \includegraphics[width=1.0\textwidth]{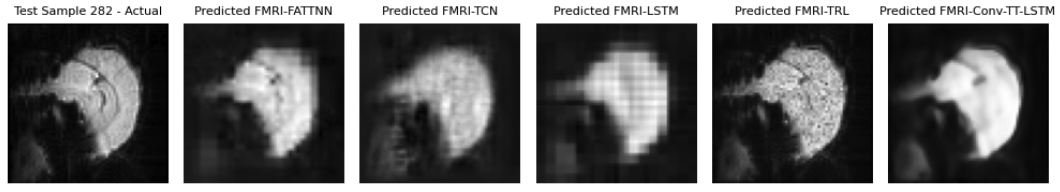} 
    \includegraphics[width=1.0\textwidth]{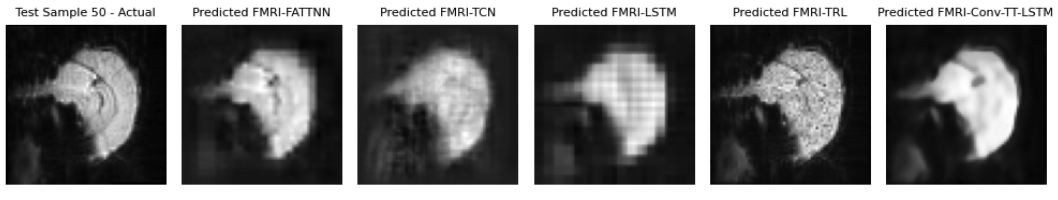} 
    \caption{Comparisons between ground-truth values and predicted outputs for selected samples in FMRI dataset. From left to right: Ground truth, FATTNN, TCN, LSTM, LSTM, TRL, and Conv-TT-LSTM,.}
\label{fig:fmri}
\end{figure*}

\end{document}